\documentclass{article}

\PassOptionsToPackage{numbers, compress}{natbib}
\usepackage[most]{tcolorbox}
\usepackage{lipsum} % For dummy text
\usepackage{amssymb} % For symbols if needed
\usepackage[preprint]{neurips_2025}

\usepackage[utf8]{inputenc} % allow utf-8 input
\usepackage[T1]{fontenc}    % use 8-bit T1 fonts
\usepackage{hyperref}       % hyperlinks
\usepackage{url}            % simple URL typesetting
\usepackage{booktabs}       % professional-quality tables
\usepackage{amsfonts}       % blackboard math symbols
\usepackage{nicefrac}       % compact symbols for 1/2, etc.
\usepackage{microtype}      % microtypography
\usepackage{xcolor}         % colors

\usepackage{amsmath}
\usepackage{amssymb}
\usepackage{mathtools}
\usepackage{amsthm}

\usepackage[capitalize,noabbrev]{cleveref}

\theoremstyle{plain}
\newtheorem{theorem}{Theorem}[section]

\newtheorem{lemma}[theorem]{Lemma}

\theoremstyle{definition}
\newtheorem{definition}[theorem]{Definition}

\theoremstyle{remark}
\newtheorem{remark}[theorem]{Remark}

\usepackage{subcaption}
\usepackage[textsize=tiny]{todonotes}
\usepackage{hyperref}
\usepackage{url}

\usepackage{float}
\usepackage{amssymb}
\usepackage{amsmath}
\usepackage{amsthm}
\usepackage{enumitem}
\usepackage{caption}
\usepackage{multirow}
\usepackage{graphicx}
\def\int{\displaystyle\mathop {\mbox{\rm int}}}    % interior

\usepackage{makecell}

\DeclareUnicodeCharacter{2212}{\ensuremath{-}}

\usepackage{algorithm,algorithmic}

\def\argmax{\displaystyle\mathop {\mbox{argmax}}}

\def\argmax{\displaystyle\mathop {\mbox{\rm argmax}}}

\def\real{\mathbb R}

\newcommand{\bbi}{{\bf I}}
\newcommand{\bbp}{{\bf P}}
\newcommand{\bbo}{{\bf O}}
\newcommand{\bba}{{\bf A}}
\newcommand{\bbb}{{\bf B}}
\newcommand{\bbq}{{{\bf Q}}}
\newcommand{\bbr}{{{\bf R}}}
\usepackage{comment}

\newcommand{\bbs}{{\bf S}}

\newcommand{\bbu}{{\bf U}}
\newcommand{\bbh}{{\bf H}}
\newcommand{\bbg}{{\bf G}}

\newcommand{\bbv}{{\bf V}}
\newcommand{\bbx}{{\bf X}}
\newcommand{\bby}{{\bf Y}}
\newcommand{\bbc}{{\bf C}}
\newcommand{\bbw}{{\bf W}}

\newcommand{\bbm}{{\bf M}}

\makeatletter
\newcommand*{\rom}[1]{\expandafter\@slowromancap\romannumeral #1@}
\makeatother

\usepackage{wrapfig}
\usepackage{multicol}

\usepackage[most]{tcolorbox}

\tcbset{
  tight eq box/.style={
    colframe=blue!80!black,
    colback=blue!5,
    boxrule=0.8mm,
    arc=0.3mm,
    boxsep=0pt,
    left=0pt,
    right=0pt,
    top=2pt,
    bottom=2pt,
  }
}
\title{SUMO: \underline{Su}bspace-Aware \underline{M}oment-\underline{O}rthogonalization for Accelerating Memory-Efficient LLM Training}

\author{%
  Yehonathan Refael \\
  Faculty of Engineering \\
  Tel Aviv University \\
  \texttt{yehonathan@tau.ac.il} \\
  \And
  Guy Smorodinsky \\
  Department of Computer science \\
  Ben Gurion University \\
  \texttt{smorodin@post.bgu.ac.il} \\
    \And
  Tom Tirer \\
  Faculty of Engineering \\
  Bar-Ilan University \\
  \texttt{tirer.tom@biu.ac.il} \\
  \And
  Ofir Lindenbaum \\
  Faculty of Engineering \\
  Bar-Ilan University \\
  \texttt{ofir.lindenbaum@biu.ac.il} \\
}

\begin{document}

\tcbset{
  myhighlight/.style={
    colback=blue!5!white, % light blue background
    colframe=blue!5!white, % no visible border
    sharp corners,
    boxrule=0pt,
    left=1mm,
    right=1mm,
    top=1mm,
    bottom=1mm,
  }
}
\maketitle

\begin{abstract}
Low-rank gradient-based optimization methods have significantly improved memory efficiency during the training of large language models (LLMs), enabling operations within constrained hardware without sacrificing performance. However, these methods primarily emphasize memory savings, often overlooking potential acceleration in convergence due to their reliance on standard isotropic steepest descent techniques, which can perform suboptimally in the highly anisotropic landscapes typical of deep networks, particularly LLMs. In this paper, we propose SUMO (\underline{Su}bspace-Aware \underline{M}oment-\underline{O}rthogonalization), an optimizer that employs exact singular value decomposition (SVD) for moment orthogonalization within a dynamically adapted low-dimensional subspace, enabling norm-inducing steepest descent optimization steps. By explicitly aligning optimization steps with the spectral characteristics of the loss landscape, SUMO effectively mitigates approximation errors associated with commonly used methods, such as the Newton-Schulz orthogonalization approximation. We theoretically establish an upper bound on these approximation errors, proving their dependence on the condition numbers of moments, conditions we analytically demonstrate are encountered during LLM training. Furthermore, we both theoretically and empirically illustrate that exact orthogonalization via SVD substantially improves convergence rates while reducing overall complexity. Empirical evaluations confirm that SUMO accelerates convergence, enhances stability, improves performance, and reduces memory requirements by up to 20\% compared to state-of-the-art methods.
\end{abstract}
\section{Introduction}

Low-rank gradient-based optimization methods have become powerful tools for reducing memory consumption during the pre-training and fine-tuning of large language models (LLMs), often without sacrificing performance, and sometimes even improving it. For instance, while pre-training LLaMA 7B typically requires around 58GB of memory, far exceeding the 24GB available on consumer GPUs like RTX 4090, recent advances, such as those discussed in \cite{zhao2024galore,refael2025adarankgrad,zhu2024apollo}, have demonstrated that LLaMA 7B can now be trained from scratch on a single 24GB GPU without the need for costly memory offloading. The theoretical analysis in \cite{zhao2024galore} attributes this efficiency to the inherent low-rank structure of the gradients, which enables optimization in a significantly reduced latent space. Furthermore, \cite{refael2025adarankgrad} found a consistent decrease in gradient rank throughout training, suggesting that low-rank optimization not only reduces memory usage but also converges toward increasingly compact subspaces.

However, despite these advancements, existing methods primarily focus on memory savings and often overlook the potential to accelerate convergence. Current approaches typically rely on standard steepest descent methods and assume isotropic geometry, which can hinder efficiency in ill-conditioned settings. This observation motivates our primary objective: to develop a subspace-aware optimizer that leverages low-rank structure while adapting to the geometry of the loss landscape. By reevaluating the choice of norm and its influence on gradient descent dynamics, we aim to design an algorithm that improves generalization, accelerates convergence,  while preserving the memory advantages of low-rank methods.

Classical gradient descent, including SGD \cite{battash2024revisiting}, performs steepest descent under the Euclidean norm, which reflects isotropic curvature. However, deep networks exhibit highly anisotropic loss landscapes, making this assumption suboptimal. Recent work shows that adaptive optimizers, such as Shampoo \citep{gupta2018shampoopreconditionedstochastictensor}, SOAP \citep{vyas2025soap}, and Muon \citep{jordan2024muon}, can be interpreted as steepest descent under non-Euclidean norms tailored to the network architecture and data structure. As shown in \cite{muoncase2024cesista}, these methods implicitly adapt to spectral or operator norms, which better capture local curvature and improve convergence. This motivates the design of subspace-aware optimizers that exploit both low-rank structure and appropriate geometry to accelerate training.

% \textcolor{red}{Tom: I think that in the following paragraph it should be stated that $\bbw$ is the weights in a layer of the network. On the other hand, Since we also talk about ADAM, let's write this for a vector w.}

To formalize the role of geometry in optimization, consider a neural network with a differentiable loss function \( \mathcal{L}: \mathcal{W} \to \mathbb{R} \) defined on a weight space \(\mathcal{W} = \mathbb{R}^n\). The local behavior around a point \( \mathbf{w} \) can be approximated by the Taylor expansion,
$\mathcal{L}(\mathbf{w} + \Delta \mathbf{w}) \approx \mathcal{L}(\mathbf{w}) + \mathbf{g}^{\top} \Delta \mathbf{w} + \frac{\lambda}{2}\|\Delta \mathbf{w}\|^2,$
where $\mathbf{g}= \nabla_{\mathbf{w}} \mathcal{L}(\mathbf{w})$, \(\lambda > 0\) captures the sharpness or curvature of the loss surface and \(\|\cdot\|\) is a chosen norm reflecting the geometry of the optimization landscape. Minimizing this approximation corresponds precisely to performing steepest descent under the given norm constraint. According to \cite{bernstein2024old}, the solution to this minimization explicitly takes the form,
\[
\Delta \mathbf{w} = -\frac{\|\mathbf{g}\|_*}{\lambda} \argmax_{\mathbf{t}:\|\mathbf{t}\|=1} \mathbf{g}^{\top}\mathbf{t},
\]
where \(\|\cdot\|_*\) denotes the dual norm of \(\|\cdot\|\). Adaptive optimizers differ primarily in their norm choices.
Adam utilizes a dynamic Max-of-Max norm constraint.
Recent optimizers consider matrix norms while applying steepest descent at the layer level.
Muon imposes a fixed Schatten-\(p\) norm constraint for large \(p\), effectively using the spectral norm on weight matrices \cite{kovalev2025,jordan2024muon}. %(Liu et al., 2025; Jordan et al., 2024). 
Shampoo \citep{gupta2018shampoopreconditionedstochastictensor} dynamically learns the optimal approximate Schatten-\(p\) norm for steepest descent, with its variants like SOAP \citep{sun2023adasamboostingsharpnessawareminimization} apply momentum to efficiently navigate the space of possible norms. Muon, by contrast, operates within a relatively fixed but large Schatten-\(p\) norm, striking a balance between the dynamic adaptability of Shampoo and the static spectral norm constraints. Since neural network weights locally act as linear operators on Euclidean spaces, the induced operator (spectral) norm provides a natural constraint that aligns with the curvature of the loss surface. This perspective motivates gradient orthogonalization, which ensures that optimization updates respect the spectral norm, thereby controlling perturbation magnitude and enhancing optimization stability and efficiency \cite{bernstein2024old}.

While norm-induced optimization methods offer a principled way to align updates with the geometry of the loss landscape, their practical deployment often incurs substantial computational overhead. For instance, Shampoo requires computing matrix inverses or root operations at every optimization step, which can be computationally expensive for large-scale neural networks. Similarly, Muon's first-order moments-orthogonalization, though effective, involves a costly approximation to spectral decompositions, computed by applying five iterations of Newton-Schulz \citep{higham1986newton} (referred to as Newton-Schulz5). Therefore, there is an inherent trade-off between the theoretical optimality provided by these norm-induced optimization approaches and their practical computational demands.

%In this paper, we first establish an upper bound on the error associated with the approximation of the Newton-Schulz moment-orthogonalization as a function of the number of iterations and the moment condition number, which analytically elucidates why the approximation becomes increasingly inaccurate and numerically unstable as the matrix is more ill-conditioned. Subsequently, we deduce a convergence rate for Muon optimization and compare it to the rate obtained by substituting Newton-Schulz estimation with a Singular Value Decomposition (SVD), which turns out to be fastest as a factor of the accumulated Newton-Schulz approximation errors. Leveraging the observation that the gradients of the layers in LLMs exhibit an initially low-rank structure, we propose a low-dimensional subspace-aware optimization approach for LLM training. This method utilizes SVD moment orthogonalization, which benefits from the relatively low computational cost associated with SVD calculations for low-rank input matrices. Additionally, we provide a convergence proof to support our approach. Finally, we present experimental results demonstrating the advantages of our method in terms of accelerated convergence and improved performance.

To bridge the gap between the geometric advantages of norm-induced methods and their computational costs, we first analyze the limitations of existing approximations. We derive an upper bound on the error introduced by the Newton-Schulz orthogonalization, demonstrating that this error increases with the condition number of the moment matrix. This finding explains the increasing instability of the Newton–Schulz5 method in ill-conditioned scenarios, which we subsequently demonstrate to occur in the first-order moment matrices during the training of large language models (LLMs). Building on this analysis, we establish a convergence rate for Muon optimization and compare it to an alternative method that replaces the Newton–Schulz approach with exact Singular Value Decomposition (SVD). Remarkably, we find that the SVD-based approach converges faster, with improvements directly proportional to the accumulated errors from the moments' orthogonalization by the Newton–Schulz5 method. Motivated by the empirical observation that gradients in LLMs often exhibit a low-rank structure, especially during early training, we propose a subspace-aware optimization scheme. This scheme performs exact SVD-based moment orthogonalization within a low-dimensional adaptive subspace. This approach benefits from the relatively low computational cost of SVD calculations for low-rank input matrices and enhances convergence stability. Also, our approach attains an even greater reduction in memory usage than all previous low-rank training methods by relying solely on the first-order moment, as detailed below in Table~\ref{tab:your_label}. We support our method with a theoretical convergence guarantee and validate its empirical benefits through experiments, demonstrating faster training and better model performance compared to existing methods.

% In this study, we suggest applying low-dimensional subspace-awareness.  

% In this paper, for the training and finetuning of LLMs, we address these challenges through conducting the orthogonalization in low-dimensional spaces. This approach involves: (1) focusing on the subspace that contains the gradient, which is characteristically low-rank, thereby circumventing computations on superfluous dimensions, and (2) ensuring that optimizer states retain the low-rank moments instead of the full representation where the extra information proves redundant.

%  To the best of our knowledge, this work is the first to (1) provide an analytical elucidation that precisely delineates the rationale for employing Newton-Schulz estimation in the orthogonalization of low-ranked gradients, and (2) furnish a convergence proof for the Muon method without simplification by excluding the cumulative error attributable to Newton-Schulz estimation. We commence with the following framework and notations.
\begin{table}[hbt]
%\vspace{-3mm}
\centering
\caption{Comparison of properties between SUMO, GaLore, Adam, Shampoo, and SOAP. Assume $\bbw \in \mathbb{R}^{m \times n}$ with $m \geq n$, a constant projection rank $r$ and a subspace update rate $K$.}
\resizebox{\textwidth}{!}{%
$\begin{array}{l|lllll}
\hline &\text{SUMO} & \text{Adam} & \text{Shampoo} & \text{SOAP} & \text{GaLore} \\
\hline 
\text{Computation}  & \fcolorbox{green}{white}{$O(m n r+m n^2/K)$} & O(m n) & O(m^3+n^3) &  O(m^3+n^3) & \fcolorbox{green}{white}{$O(m n r+m n^2/K)$} \\
\text{Optim.~states memory}  & 
\fcolorbox{green}{white}{$n r + m r$}
& 2 m n & m^2+n^2 & 2 m n+2 m^2+2 n^2 & 2 n r+m r \\
\text{Subspace-aware} &\textcolor{green}{\checkmark} & \textcolor{red}{\times} & \textcolor{red}{\times} & \textcolor{red}{\times} &  \textcolor{green}{\checkmark}\\
\text{Orthogonalization} &\textcolor{green}{\checkmark} & \textcolor{red}{\times} & \textcolor{red}{\times} & \textcolor{red}{\times} &  \textcolor{red}{\times}\\
\hline
\end{array}$
}
\label{tab:your_label}
\vspace{-10pt}
\end{table}

\section{Related Work}
\paragraph{Low-rank gradient optimization.}
Low-rank gradients naturally emerge during neural network training, as shown in both theoretical and empirical studies \cite{zhaoZerOInitializationInitializing2022, cossonLowRankGradientDescent2023, yang2023spectral}. This structure has been leveraged to reduce memory and computational costs during training \cite{gooneratneLowrankGradientApproximation2020, huangLowRankGradientDescent2023, modoranuErrorFeedbackCan2024}. Recent work \cite{refael2025adarankgrad} showed that gradients in reversible layers \cite{tian2021} tend to collapse to rank one over time and used this to adaptively adjust gradient rank in Adam. In this paper, we demonstrate that the same low-rank trend is present in the first-order moment, which we utilize to efficiently apply exact orthogonalization—avoiding the accumulation of approximation errors, such as those encountered in Newton-Schultz, during optimization.

\paragraph{Memory efficient optimizers.}
Reducing the memory demands of training large language models (LLMs) has driven extensive algorithmic research. One research direction, initiated by LoRA \cite{hu2021lora}, reduces the number of trainable parameters via low-rank adaptation. Yet, such methods often fall short of fully parameterized models, especially during pre-training. Another direction does not restrict the set of trainable parameters but instead optimizes the training methods, with notable examples including AdaRankGrad, GaLore, Fira, Flora, Adam-mini, GaLore-mini, LDAdam, GoLore, LoQT, and Apollo \cite{zhao2024galore, refael2025adarankgrad, zhu2024apollo, chen2024fira, Hao2024FloraLA, zhang2024adam, robert2025ldadama, loeschcke2024}, integrating low-rank gradient projections in optimization. In this work, we reduce memory usage even further by relying solely on a first-order momentum, as shown in Table \ref{tab:your_label}.

\paragraph{Gradient preconditioning.}
% Gradient preconditioning is key to improving optimizer performance. Existing methods include signed gradients \cite{bernstein2018signsgd, crawshaw2022robustness, lion, kunstner2023noise}, gradient clipping \cite{zhang2020adaptive}, normalization \cite{zhang2020adaptive, you2019lamb}, and whitening \cite{yang2008principal, adam, hwang2024fadam, jordan2024muon, bernstein2024oldoptimizernewnorm, bernstein2024modulardualitydeeplearning, carlson2015preconditioned}. Recent work \cite{jordan2024muon, tuddenham2022orthogonalising} highlights the benefits of gradient orthogonalization, which stabilizes training by enforcing uniform update directions. Unlike these methods that operate in full parameter space, our approach performs exact SVD-based orthogonalization within a dynamically chosen low-rank subspace, yielding greater stability and reduced computational cost.
Preconditioning the Gradient method is critical in enhancing the efficiency and effectiveness of optimizers. Several notable approaches for using a preconditioner have emerged, including methods based on signed gradients~\citep{bernstein2018signsgd, crawshaw2022robustness, lion, kunstner2023noise}, gradient clipping~\citep{zhang2020adaptive}, normalization~\citep{zhang2020adaptive, you2019lamb}, and gradient whitening~\citep{jordan2024muon, yang2008principal, adam, hwang2024fadam, bernstein2024oldoptimizernewnorm, bernstein2024modulardualitydeeplearning, carlson2015preconditioned}. Recent studies~\citep{jordan2024muon, tuddenham2022orthogonalising} explored gradient-orthogonalization strategies, thereby speeding up training. Orthogonalizing gradients effectively constrains updates to lie on directions of uniform magnitude (spectral radius = 1), preventing updates from exaggerating specific gradient directions over others.
This procedure ensures a form of normalization that mitigates potential instabilities from ill-conditioned gradients. Unlike these methods, which apply preconditioning or approximate orthogonalization in the high-dimensional parameter space, our approach performs exact SVD-based orthogonalization within an adaptively selected low-rank subspace, offering improved stability and lower computational overhead.

%{\color{red}TODO - saying something about our paper}
%layer-wise global normalization applied to Adam's internal states in Lamb~\citep{you2019lamb}, our normalization is directly applied row-wise to raw gradients. Compared to Muon~\citep{jordan2024muon}, our method entirely eliminates first-moment tracking, thus achieving complete statelessness by relying exclusively on row-wise normalization. Additionally, we introduce an efficient heuristic for computing the inverse square root operation, achieving computational complexity of $\mathcal{O}(m^2)$ instead of the standard $\mathcal{O}(m^3)$ complexity found in existing Newton-Schulz-based approaches~\citep{song2022fast, li2018towards, huang2019iterative}.
% \paragraph{OSGDM} 

\paragraph{Orthogonal Stochastic Gradient Descent with Momentum (OSGDM).}
OSGDM \cite{tuddenham2022orthogonalising} is a recently introduced first-order optimization method that speeds up neural network training by orthogonalizing gradients before the optimization step. Specifically, for a data batch $\xi^{(t)} $, OSGDM applies SVD to the gradient matrix $\bbg_l^{(t)}=\nabla_{\bbw_l} \mathcal{L}\left(\Phi\left(\xi^{(t)} ; \mathbf{\theta}\right)\right)$ of each neural network layer $l$ to generate an orthonormal gradient approximation $\bbo_l$. This ensures diversity among learned representations and reduces redundancy. The update rule for OSGDM with momentum term $\gamma$ and learning rate $\eta$ is defined as, $$\mathbf{O}_l^{(t)}=\operatorname{orth}\left(\bbg_l^{(t)}\right),\quad \bbm^{(t+1)}_l \leftarrow \gamma \bbm^{(t)}_l + \eta \mathbf{O}_l^{(t)}, \quad \bbw^{(t+1)}_l \leftarrow \bbw^{(t)}_l - \bbm^{(t+1)}_l,$$ where 
$\operatorname{orth}(\bbg)=\left(\bbg \bbg^{\top}\right)^{-1/2} \bbg $ is the ortogonalization operator, and $\bbm_l$ is the first order moment of layer $l$. Despite the additional computational overhead of SVD, OSGDM empirically converges faster and achieves higher accuracy than common methods such as Adam.

\paragraph{Muon optimizer.}
At iteration \(t\), given weight \(\mathbf{W}^{(t)}\), momentum \(\mu\), learning rate \(\eta_t\), and objective \(\mathcal{L}_t\), Muon, introduced by \cite{jordan2024muon}, constructs the update rule, $$\bbm^{(t)} = \mu \bbm^{(t-1)} + \bbg_l^{(t)},\quad
\mathbf{O}_l^{(t)}= \text{Newton-Schulz5}(\bbm^{(t)}),\quad
\mathbf{W}^{(t+1)}= \mathbf{W}^{(t)} - \eta^{(t)}\mathbf{O}_l^{(t)}.$$ Here, \(\bbm^{(t)}\) is the momentum at iteration \(t\), initialized as a zero matrix when \(t=0\). The Newton-Schulz5 method \cite{bernstein2024oldoptimizernewnorm} approximates \((\bbm^{(t)}{\bbm^{(t)}}^{\top})^{-1/2}\bbm^{(t)}\), orthogonalizing \(\bbm^{(t)}\) and thus ensuring uniform update directions, avoiding dominance by few directions. Muon explicitly controls the norm of gradient updates—particularly the spectral norm (or Schatten-$p$ norm with large $p$), which limits updates to smaller, well-conditioned steps in parameter space. By constraining the spectral norm, moment orthogonalization implicitly prevents overly large or ill-conditioned parameter updates. Such updates often lead to poor generalization due to instability or overfitting. Shortly after the introduction of Muon, the study in \cite{liu2025muonscale} proposed a framework to scale Muon for larger LLMs, mainly adding weight decay, and
carefully adjusting the per-parameter update scale.    

% \textbf{Although the Newton-Schulz5 orthogonalization-approximation is employed, the computational complexity remains substantial, . }

\section{Method and Main Results}\label{sec:main_results}

% \begin{tcolorbox}[myhighlight]
% \textbf{Question:} \textit{Can the natural low-rank structure of LLMs gradient can be utilize to }
% \end{tcolorbox}

\subsection{Theoretical Motivation: Exact moments orthogonalization leads to significantly faster convergence}

Previous work on pre-training and fine-tuning large language models (LLMs) has primarily focused on reducing memory usage for constrained hardware or lowering computational cost (e.g., \cite{zhao2024galore,refael2024,zhu2024apollo}). In this paper, we take a step toward accelerating LLM optimization by showing that applying exact orthogonalization (e.g., via SVD) to the first-order moment offers a practical advantage, even over the most accurate approximations, such as the commonly used Newton-Schulz5 method. Specifically, we find that SVD converges faster and incurs lower computational overhead. To support this, we first present a new observation: the moment matrix in LLM training tends to decrease in rank over time. Building on this, we then derive an upper bound on the approximation error of Newton-Schulz5, showing that it depends on both the number of iterations and the matrix condition number, highlighting its limitations in ill-conditioned or low-rank settings (which is precisely the case in LLM optimization moments). This motivates the need for more accurate orthogonalization of moment matrices during LLM training. Of course, applying SVD directly to full-sized layers is generally impractical. The surprising result, however, is that when integrated into a low-rank optimization scheme, the use of SVD becomes not only feasible but preferable. We conclude with a convergence analysis of Muon optimization, which, under these conditions, converges significantly more slowly than the SVD-based alternative. To the best of our knowledge, our convergence analysis of Muon optimization is the first to avoid neglecting the error in the Newton-Schultz approximation \cite{li2025note}. The proofs of all lemmas and theorems of this section are relegated to the Appendix~\ref{apd:main_results}.

\begin{lemma}[Moment Becomes Low-Rank During Training]
Let $\bbm^{(t)} \in \mathbb{R}^{n \times m}$ denote the first moment of a reversible layer\footnote{Reversible networks are formally defined in Appendix~\ref{Reversibility}} in a moment-based optimization algorithm, updated according to
$\bbm^{(t)} = \beta_1 \bbm^{(t-1)} + \bbg^{(t)},$
where $\bbg^{(t)}$ is the gradient matrix at iteration $t$. Let $\bbm^{(t)} = \bbu^{(t)} \Sigma^{(t)} {\bbv^{(t)}}^\top$ be the singular value decomposition (SVD) of $\bbm^{(t)}$, and define the rank-$r$ orthogonal projection matrix as $\bbp^{(t)}(r) = \bbu^{(t)}[:, 1\!:\!r]\,\bbu^{(t)}[:, 1\!:\!r]^\top$. Then the relative error of the best rank-one approximation,
\begin{align}
    \kappa_M(t) \triangleq \frac{\|\bbm^{(t)} - \bbp^{(t)}(1)\bbm^{(t)}\|_F^2}{\|\bbm^{(t)}\|_F^2},\label{moment_decay}
\end{align}
satisfies $\kappa_M(t) \leq O(C^{-t})$ for some constant $C > 1$. \label{lem::moment_lowrank}
\end{lemma}
The above result, in (\ref{moment_decay}), implies that $\bbm^{(t)}$ approaches its rank-one approximation $\bbp^{(t)}(1) \bbm^{(t)}$, as the iteration number increases, namely, $\bbm^{(t)}$ becomes rank-one. The following Lemma~\ref{lma:estimating_error} characterizes the impact of the moments' low-rank structure on the approximation error of the Newton-Schulz5 orthogonalization.

% As a motivation for the next Lemma \ref{lem::error_bound}, in which we analyze the orthogonalization-approximation error of Newton-Schulz5, as it occurs, for example in the Muon optimizer, as a function of the number of conditions, we emphasize that even the though projection $\bbm^{(t)}$ on its , although from the inherent low-dimensional structure, still does not limit the values of the number of conditions.

% In the following figure empirically illustrates the condition number of the low rank moment in training Roberta base model on Glue RTE dataset using Galore optimizer. 

\begin{lemma}[Orthogonalization error $\mathbf{\mathcal{E}}_{i}$]\label{lma:estimating_error}
    For a matrix $\bba \in \real^{m \times n}$, let $\sigma_1$ be the largest singular value of $\bba \bba^\top$ and $\sigma_m$ be the smallest (without the loss of generality, assume $m\leq n$). Let $r\leq m$ be the largest index where $\sigma_r > \sigma_{r+1} = \dots = \sigma_m \geq 0$. Let $\kappa = \frac{\sigma_1}{\sigma_m}$ by the condition number of $\bba \bba^\top$.  Denote $\mathbf{\mathcal{E}}_{i}$ the error of Newton-Schultz after $i$ iterations. Then we have
    \begin{align}
        \|\mathbf{\mathcal{E}}_{i}\|_F \leq \sqrt{r} \cdot  \left(1 - \frac{1}{\kappa} \right)^{2^i}\label{lem::error_bound}.
    \end{align}
\end{lemma}

According to the lemma, the approximation error grows exponentially with the condition number. Given the low-rank structure of the first-order moments, low-dimensional optimization can mitigate this error. Specifically, projecting the moment estimates $\hat{\bbm}^{(t)}$ onto their dominant (small) $r$-dimensional subspace ensures that the squared moment $\hat{\bbm}^{(t)} \hat{\bbm}^{(t)\top}
$ is constructed using only the top $r$ squared eigenvalues. These dominant components are significantly larger and exclude near-zero values, resulting in a substantially lower condition number compared to that of the full-rank squared moment matrix. This observation motivates the use of the Muon optimizer within a low-rank optimization framework for LLMs, including 2D reversible layers. Such an approach not only preserves the inherent memory efficiency of low-rank methods but also reduces approximation error during optimization, potentially leading to faster convergence and improved performance compared to full-dimensional training. However, we also empirically observe that the eigenvalues of the moment matrix decay gradually. As shown in Figure~\ref{fig::high_condition_plus_ev_decay}, even when projecting onto the dominant subspace, the resulting matrix $\hat{\bbm}^{(t)} \hat{\bbm}^{(t)\top} $, composed of the top $r=16$ squared eigenvalues, can still exhibit a large condition number, thereby introducing non-negligible approximation error.

\begin{figure}[htb]
\vskip -0.1 in
\centering
\begin{subfigure}[t]{0.48\linewidth}
    \includegraphics[width=\linewidth]{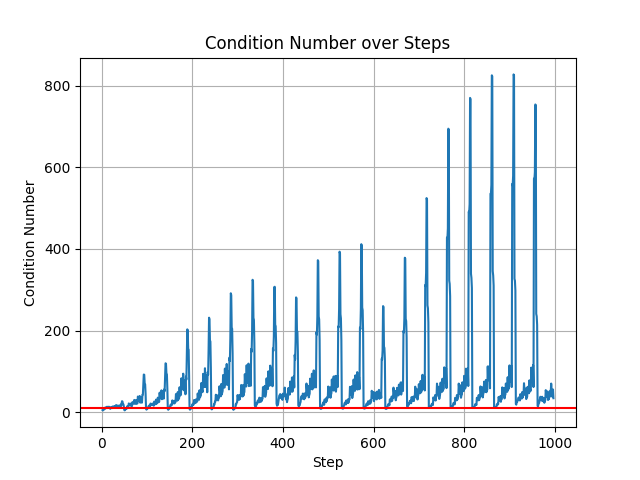}
    \caption{Condition number of the first-order moment vs. training step. The red line marks value 10.}
    \label{condtion_number_curve}
\end{subfigure}
\hfill
\begin{subfigure}[t]{0.48\linewidth}
    \includegraphics[width=\linewidth]{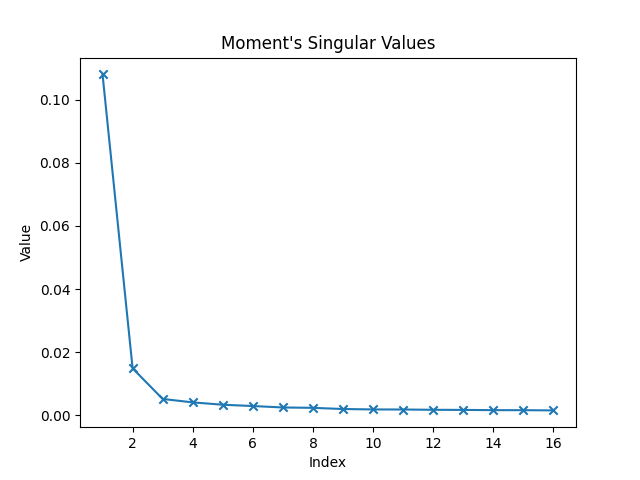}
    \caption{Illustration of the moment's singular value decay, taken arbitrarily at step 100.}
    \label{fig:svd_values}
\end{subfigure}
\caption{Evidence of anisotropy and ill-conditioning in the first-order moment matrix as a function of the Galore steps of the Roberta-base model \citep{liu2019roberta} on the GLUE dataset RTE task \citep{wang2019superglue}: (a) condition number growth, (b) spectral decay of moment.}\label{fig::high_condition_plus_ev_decay}
\vskip -0.1 in
\end{figure}

% \begin{example}[Slow convergence due to large \(\kappa\)]
% Consider the error bound (\ref{lem::error_bound})
% with \(\kappa = 1000\), $r = 9$  and \(i = 5\). Substituting these values results in
% \[
% \varepsilon^{(i=5)} \leq \sqrt{9} \cdot \left(1 - \frac{1}{1000} \right)^{32}
% = 3 \cdot\left( 0.999 \right)^{32}
% \approx 3 \cdot 0.9686.
% \]
% This means that despite performing five iterations and leveraging exponential convergence, the error remains close to the initial scale due to the large condition number \(\kappa = 1000\). This illustrates that poor conditioning can significantly slow down convergence, even when the iteration count is relatively high.
% \end{example}
To comprehend the cumulative error of Newton-Schulz5 orthogonalization at each optimization step, we proceed to derive the convergence rate of the Moun optimization. To that end, we now provide some notations. Consider a neural network denoted as $\Phi(\cdot;\boldsymbol\theta)$, which consists of $L$ layers and is parameterized by {\footnotesize{$\boldsymbol{\theta} \triangleq \left[\bbw_1^{d_1 \times d_0}, \ldots, \bbw_{L-1}^{d_{L-1} \times d_{L-2}}, \bbw_{L}^{d_{L} \times d_{L-1}}\right]$}}. Here, $\bbw_{i}$ represents the weights tensor parameters associated with the $i$-th layer, for $i \in [L]$. We denote the differential loss $\mathcal{L}$, where, with a slight abuse of notation, we write the training problem by $\min_{\mathbf{W}} \ \mathcal{L}(\mathbf{W}) = \mathbb{E}_\xi[\mathcal{L}(\Phi(\mathbf{W}, \xi))],$ if the context refers to the weights of a certain layer. We use the Frobenius norm, denoted \(\| \cdot \|_F\), which is induced by the inner product \(\langle \bbx, \bby \rangle = \operatorname{tr}(\bbx^\top \bby)\). Assume that the stochastic gradient \(\nabla \mathcal{L}(\mathbf{W}, \xi)\) is an unbiased estimator of the full gradient \(\nabla \mathcal{L}(\mathbf{W})\), with variance bounded by \(\sigma^2\), i.e., $\mathbb{E}[\| \nabla \mathcal{L}(\mathbf{W}, \xi) - \nabla \mathcal{L}(\mathbf{W}) \|_F^2] \leq \sigma^2.$
Let $\mathbf{\mathcal{E}}^{(t)}_i= \text{orth}(\bbm^{(t)}) - \text{Newton-Schulz}(\bbm^{(t)})$ denote the approximation error of the Newton-Schulz (with $i\geq1$ iteration) at time $t$, where \(\bbm^{(t)}\) denotes the moment at iteration \(t\).

% This result has two important implications: (1) As the condition number increases, the approximation error also increases (in an exponential manner). In addition, given that LLM layer gradients naturlly exhibit a low-rank structure~\cite{refael2025adarankgrad,zhao2024galore}, any optimization method that does not explicitly exploit this structure, anyways, in practice, operates within a low-dimensional subspace. Therefore, applying the Newton–Schultz orthogonalization approximation within this intrinsic low-dimensional subspace results in significantly improved accuracy. In particular, such subspace-aware optimization can entirely eliminate the impact of zero or near-zero eigenvalues on the approximation error cused by the full-ranked ill-conditioned gradients.
% (2) The number of iterations $k$ required to reduce the Newton–Schultz approximation error—cused by an ill-conditioned gradient—suggests that a precise approximation can be achieved with a comparable or even smaller computational cost, especially when accounting for the structure of the gradient.

\allowdisplaybreaks

\begin{lemma}[Exact convergence rate of Muon] \label{lma:moun_convergence}
Consider the Muon optimizer update w.r.t layer $\bbw\in\mathbb{R}^{m\times n}$ defined by
\begin{align*}
& \mathbf{M}^{(t)} \leftarrow \beta \mathbf{M}^{(t-1)}+(1-\beta) \mathbf{G}^{(t)} \\
& \mathbf{O}^{(t)} \leftarrow \mathbf{U}^{(t)} \mathbf{V}^{(t)^{\top}}+\mathcal{E}_i^{(t)}, \quad(i \text { iterations Newton-Schulz approximation })\\
&
\mathbf{W}^{(t+1)} \leftarrow \mathbf{W}^{(t)}-\eta_t \mathbf{O}^{(t)},
\end{align*}
where $\mathbf{M}^{(t)}=\mathbf{U}^{(t)} \mathbf{S}^{(t)} \mathbf{V}^{(t)^{\top}}$ denotes the singular value decomposition of $\mathbf{M}^{(t)}$, and $\mathcal{E}_i^{(t)}$ represents the Newton-Schulz5 approximation error after $i$ iterations. Suppose the following: \begin{itemize}
        \item The gradient \(\nabla \mathcal{L}(\mathbf{W})\) is \(L\)-Lipschitz continuous.
        \item There exists \(\delta > 0\) such that \(\|\mathbf{\mathcal{E}}^{(t)}_i\| \leq \delta \|\mathbf{U}^{(t)} {\mathbf{V}^{(t)} }^\top\| = \delta \sqrt{n}\), for all \(t\).
    \end{itemize}  
If we take $\beta = 1 - \alpha$ with $\alpha = \min(\frac{\sqrt{RL}}{\sigma\sqrt{T}}, 1)$, $\eta_t = \eta = \frac{\sqrt{4R}}{\sqrt{(10/(1-\beta) + 2m + 4m\delta + 2m\delta^2)TL}}$, and $B=1$ (batch free convergence) than $\frac{1}{T}\sum_{t=1}^{T} \mathbb{E}\left[\|\nabla \mathcal{L}(\mathbf{W}^{(t)})\| \right]$ is bounded by\allowdisplaybreaks
   {\footnotesize{ $$
 \mathcal{O} \left( \left[\frac{\sqrt{RLm(2 + 4\delta + 2\delta^2)}}{\sqrt{T}} + \frac{\sigma^2}{\sqrt{RLT}} + \frac{\sigma(RL)^{1/4}+\sqrt{\sigma}(RL)^{1/4}}{T^{1/4}} \right] \frac{1}{1 - 4\sqrt{m} \delta}\right), $$}} 
    where $R=\mathcal{L}(\mathbf{W^{(0)}}) - \mathcal{L}^*$.
    If we take $\beta$ as an arbitrary constant, and $B=T$, we have,{\footnotesize{\begin{align*}\frac{1}{T}\sum_{t=1}^{T} \mathbb{E}\|\nabla \mathcal{L}(\mathbf{W}^{(t)})\| \leq \mathcal{O}\left(\left[\frac{\sqrt{RLm(2 + 4\delta + 2\delta^2)}}{\sqrt{T}} + \frac{\sqrt{RL}}{\sqrt{T}} + \frac{\sigma}{T^{3/2}} + \frac{\sigma}{\sqrt{T}}\right]\frac{1}{1 - 4\sqrt{m} \delta}\right).\end{align*}}}
    \end{lemma}

\begin{remark}[Comparison: slower convergence vs exact orthogonalization] When $\delta = 0$, indicating an absence of error, the convergence rate is aligned with the one derived in \citep{li2025note}, Theorem~2.1, that is {\footnotesize{$$\frac{1}{T}\sum_{t=1}^{T} \mathbb{E}\left[\|\nabla \mathcal{L}(\mathbf{W}^{(t)})\| \right] \leq \mathcal{O}\left(\frac{\sqrt{n R L}}{\sqrt{T}}+\frac{\sigma}{T^{3 / 2}}+\frac{\sigma}{\sqrt{T}}\right).$$}} This result overlooks the error associated with the Newton-Schulz5 approximation because it is based on a theoretically exact method of orthogonalization.
\end{remark}

\begin{remark}[The impact of $\delta$ on the convergence rate]
A reduction in $\delta$ is associated with an improvement in the convergence rate. Furthermore, it should be noted that $\delta$ influences the step size $\eta$; a larger $\delta$ yields a smaller $\eta$, providing an additional explanation for the convergence rate.
\end{remark}

\begin{remark}[The size of $\delta$]
We acknowledge that the findings of our analysis are applicable only under the conditions specified in $1 - 4 \sqrt{n} \delta > 0 \Rightarrow \delta < \frac{1}{4 \sqrt{n}}$. In scenarios where $\delta > \frac{1}{4 \sqrt{n}}$ applies, the algorithm may fail to converge. To ensure that $\delta$ remains sufficiently small, the Newton-Schulz5 method necessitates a substantial number of iterations, consequently slowing down the convergence.
\end{remark}

\begin{remark}[Speed-up by SVD vs Newton-Schulz5 approximation]
According to Lemma~\ref{lma:estimating_error}, these low-rank moments, which inherently possess exceptionally high $\kappa$, result in an error expressed by $(1-\varepsilon)^{2^i}$ concerning a remarkably small $\varepsilon$. This situation necessitates numerous iterations for the Newton-Schultz method to converge. For example, if $(1-\varepsilon) = 0.99$ is considered and Newton-Schultz5 is utilized with 5 iterations, the error would be $\approx 0.99^{32} = 0.725,$ relative to the norm of the moment, namely $\bbm$.
\end{remark}

Recall that in the low-rank setting, accurately computing the pseudoinverse using singular value decomposition (SVD) is numerically advantageous and reasonably computationally affordable compared to iterative methods such as Newton--Schulz. For a general matrix $\mathbf{A} \in \mathbb{R}^{n \times m}$, the SVD provides a decomposition $\mathbf{A} = \mathbf{U}\mathbf{\Sigma}\mathbf{V}^\top$, with $\mathbf{U} \in \mathbb{R}^{n \times n}$, $\mathbf{\Sigma} \in \mathbb{R}^{n \times m}$, and $\mathbf{V} \in \mathbb{R}^{m \times m}$. The Moore--Penrose pseudoinverse is then calculated as $\mathbf{A}^\dagger = \mathbf{V}\mathbf{\Sigma}^\dagger\mathbf{U}^\top$, requiring approximately $4nm^2 + 8m^3$ floating-point operations (FLOPs) for the initial decomposition, and an additional $mn^2 + m^2n$ FLOPs for subsequent multiplications, totaling roughly $4nm^2 + 8m^3 + mn^2 + m^2n$ FLOPs.

Alternatively, approximating the inverse of $\mathbf{A}^\top \mathbf{A} \in \mathbb{R}^{m \times m}$ using Newton--Schulz iterations involves $nm^2$ FLOPs to form the matrix $\mathbf{A}^\top \mathbf{A}$, approximately $20m^3 + 10m^2$ FLOPs for five iterations, and an additional $m^2n$ FLOPs to multiply by $\mathbf{A}^\top$, resulting in a total of about $nm^2 + m^2n + 20m^3 + 10m^2$ FLOPs. For example, when the rank is $m = 8$ and $n = 1024$, the SVD approach requires approximately twice as many operations as Newton--Schulz5. Nonetheless, given the superior numerical stability and inherent optimality of the SVD-based method, this moderate increase in computational effort remains acceptable, especially when accuracy and stability are prioritized.

\subsection{Method}
We are now ready to present our main algorithm designed to accelerate the low-rank optimization scheme outlined in Algorithm~\ref{alg::SUMO}. A detailed mathematical formulation of the weight update rule proposed in this paper is given in Appendix~\ref{Update_Step}. The algorithm consists of four primary blocks, all contained within an outer loop that continues until convergence is achieved or a predefined number of epochs is reached. Each block serves a specific purpose, which will be explained in detail below.

\begin{itemize}[leftmargin=*]
    \item \textbf{Block 1 (Adaptive Subspace Selection)}: We select the subspace along the directions of the $r$ largest eigenvectors, but since computing full SVD for large matrices is computationally intensive and memory-demanding, we leverage the Randomized-SVD by \citep{halko2010}, which is an efficient technique for producing a ``good" proxy for the optimal low-rank approximation. It solves the optimization problem  $\arg\min _{\mathbf{Q}^{(t)} \in \mathbb{R}^{m \times r}} \left\|\bbg - \mathbf{Q}^{(t)} {\mathbf{Q}^{(t)}}^\top \bbg{(t)} \right\|_F,$
    and approximates the matrix $\bbg$ as $\hat{\bbg} \approx \mathbf{Q}^{(t)} {\mathbf{Q}^{(t)}}^\top \bbg{(t)},$ that requires $O(mnr + mr^2)$ operations, instead of $O(\min(mn^2, m^2n))$ applied by SVD.
    \item \textbf{Block 1.1 (Moment Subspace Transformation)}: We transform the first-order moments evaluated during the low-rank optimization steps, which occur in Block 2, between the preceding and the newly updated subspace. This transformation is required because, as will be demonstrated later, the first moments of the gradients in Block 2 are aligned with the previously projected subspace. Consequently, a transformation is necessary to translate them from the former subspace to the current one. 
    \item \textbf{Block 2 (Low-Rank Steepest Descent Optimization)}: Here we calculate the (steepest) optimization step. SVD operation is adopted to solve exactly $(\bbm^{(t)}{\bbm^{(t)}}^{\top})^{-1/2}\bbm^{(t)}$.
    Let $\bbu \boldsymbol{\Sigma} \mathbf{V}^{\top}=\hat{\bbm}^{(t)}$ 
    be the singular value decomposition (SVD) of $\hat{\bbm}^{(t)}$, we will have $(\bbm^{(t)}{\bbm^{(t)}}^{\top})^{-1/2}\bbm^{(t)}=\mathbf{U} \mathbf{V}^{\mathbf{T}}$, which orthogonalizes $\hat{\bbm}^{(t)}$. Formally, the 
    $$\text{Orthogonalization\_SVD}(\bba)=\underset{\bbo}{\operatorname{arg\,min}} \{ \|\bbo - \bba\|_F : \text{either } \bbo^T \bbo = \bbi \text{ or } \bbo\bbo^T = \bbi \}.$$
    \item \textbf{Block 3}: Rather than using standard gradient clipping, we adopt the Norm-growth Limiter (NL) introduced in \cite{chen2024fira}, which has been shown to slightly outperform traditional clipping techniques by better constraining the progression of gradient magnitudes. Specifically, the gradient update is modified as follows,
    ${\small{
    \text{if } \frac{\left\|\bbo^{(t)}\right\|}{\left\|\bbo^{(t-1)}\right\|} > \gamma \text{ then } \bbo^{(t)} \leftarrow \frac{\bbo^{(t)}}{\left\|\bbo^{(t)}\right\|} \cdot \gamma \left\|\bbo^{(t-1)}\right\|,
    }}$
    where the scalar $\gamma$ serves as a growth threshold to regulate abrupt increases in gradient norm from one iteration to the next. We use $\gamma=1.1$, which empirically yields the best results. 

\item \textbf{Block 4 (Update Step in the Original Space).} 
To better exploit already computed gradient information, we suggest to use the orthogonal term of the gradient that lies outside the low-rank subspace spanned $\bbo^{(t)},$ namely ${\bbg^{(t)}}^\perp = \bbg^{(t)} - {\bbq^{(t)}} \underbrace{{\bbq^{(t)}}^\top \bbg^{(t)}}_{\hat{\bbg}^{(t)} }$. Importantly, ${\bbg^{(t)}}^\perp$ does not interfere with the orthogonalized moment estimation $\bbo^{(t)}$, meaning it can be incorporated without compromising stability.  Since $ \bbg^{(t)}$ is already computed and stored in each iteration, no extra memory is required. Furthermore, because $ \bbq^{(t)}$ is of a low rank (typically rank $4$, $8$, or $16$), the additional computationations overhead is negligible.
For efficient memory usage, instead of explicitly forming the perpendicular part 
$\bbg^{(t)\perp}$,
we work only in the rank-$r$ subspace spanned by ${\bbq^{(t)}}$.
Practically, we use %the already calculated projected gradient
%$\hat{\bbg}^{(t)}:=\; {\bbq^{(t)}}^\top \bbg^{(t)}$
%and 
% use the subspace–fused direction
% $\bbg_{\text{step-direction}}^{(t)}
% \;:=\; \bbg^{(t)} - \bbq^{(t)}\!\bigl(\hat{\bbg}^{(t)}-\bbo^{(t)}\bigr)
% \;=\; \underbrace{\bigl(\bbg^{(t)}-\bbq^{(t)}{\bbq^{(t)}}^\top\bbg^{(t)}\bigr)}_{\bbg^{(t)\perp}}
% \;+\; \bbq^{(t)}\bbo^{(t)} .$
fact that
$$
\underbrace{\bigl(\bbg^{(t)}-\bbq^{(t)}{\bbq^{(t)}}^\top\bbg^{(t)}\bigr)}_{\bbg^{(t)\perp}}
\;+\; \bbq^{(t)}\bbo^{(t)}
= \bbg^{(t)} - \bbq^{(t)}\!\bigl(\hat{\bbg}^{(t)}-\bbo^{(t)}\bigr),
$$
which utilizes the already calculated projected gradient
$\hat{\bbg}^{(t)}:=\; {\bbq^{(t)}}^\top \bbg^{(t)}$.
Accordingly, the pre-trained model parameters are updated
along with weight decay,
$$
\bbw^{(t)}
\;\leftarrow\;
\bbw^{(t-1)}
\;-\;\alpha\eta\,
\Bigl(\bbg^{(t)}-\bbq^{(t)}\bigl(\hat{\bbg}^{(t)}-\bbo^{(t)}\bigr)\Bigr)
\;-\;\eta\,\lambda\,\bbw^{(t-1)}.
$$
To ensure stable training across parameter matrices of different shapes, we interpret the root mean square (RMS) magnitude of updates as implicit \emph{layer-wise learning rate adaptation}, following the approach in \cite{liu2025muonscale}. By scaling updates by $\sqrt{\max(m,n)}$, our method compensates for shape-induced magnitude differences, achieving consistent effective learning rates across layers, similar to adaptive optimizers like AdamW.
\end{itemize}
\vskip -0.3 in
\definecolor{highlightcolor}{rgb}{0.9, 0.9, 0.9}  % 
\begin{algorithm}[htb]
   \caption{SUMO: Subspace-Aware Moment-Orthogonalization Optimization}
 \begin{algorithmic}
   \STATE {\bfseries Input:} A weight matrix $\textbf{W} \in \mathbb{R}^{m \times n}$ with $m \geq n$. Step size $\eta$, scale factor $\alpha$, decay rates \{$\beta_1, \beta_2$\}, weight decay $\lambda$, rank $r$, subspace update frequency $K$, small number $k\in\mathbb{N},$ step clipping ratio $\gamma$.
   \STATE \textbf{Initialize}: $t \gets 0$  
   \REPEAT
   \STATE \textcolor{blue}{\# Block 1: Calculate low rank gradient projection.}
    \STATE Sample mini-batch $B=\left\{\xi_1, \xi_2, \ldots, \xi_{{|B|}}\right\}$
    \STATE Compute $\bbg^{(t)} \leftarrow \sum_{i=1}^{|B|}\frac{\partial}{\partial \bbw} \mathcal{L}\left(\Phi(x_i, \boldsymbol{\theta}),y_i\right)$
   \IF{$t \bmod K = 0$
   } 
    \STATE $\mathbf{Q}^{(t)} \leftarrow\operatorname{Truncated\_Randomized\_SVD}(\mathbf{G}^{(t)})$  {\color{gray}\# Alternatively $\operatorname{Truncated\_SVD}(\mathbf{G}^{(t)})$}
    \STATE {\color{blue}\# Block 1.1: Moment subspaces transformation}
    \STATE $\bbr^{r \times r} \leftarrow {\bbq^{(t)}}^{\top}\bbq^{(t-1)} \text{ if } t\geq 1, \text{ else } \mathbf{0}^{r \times r}$
    \STATE ${\bbm^{(t)}}^{r \times n} \leftarrow \bbr {\bbm^{(t-1)}},\text{ if } t\geq 1, \text{ else } \mathbf{0}^{r \times n}$\hfill{\color{gray}\COMMENT{$1^{st}$-order moment}}
   \ENDIF \; {\color{gray}\# Alternatively criteria $\|\hat{\bbg}^{(t)}\|\leq\varsigma$}
    \STATE $\hat{\bbg}^{(t)}\leftarrow {\bbq^{(t)}}^{\top}\bbg^{(t)}$
    \STATE \textcolor{blue}{\# Block 2: Low-rank steepest-decent step (moment ortogonalization)}
    \STATE $\bbm^{(t)}\leftarrow \mu \bbm^{(t-1)}+ \hat{\bbg}^{(t)}$
    \STATE $\bbo^{(t)}\leftarrow \text{Orthogonalization\_SVD}(\bbm^{(t)})$
    % \STATE $\text{Update parameters } \theta^{(t)}&\leftarrow \theta^{(t-1)}- \eta \bbo^{(t)}$
   \STATE \textcolor{blue}{\# Block 3 (Optional): }
   \STATE \textbf{if} $\frac{\left\|\bbo^{(t)}\right\|}{\left\|\bbo^{(t)}\right\|}>\gamma$ \textbf{then} $\mathbf{O}^{(t)} \leftarrow \frac{\mathbf{O}^{(t)}}{\left\|\mathbf{O}^{(t)}\right\|} \cdot \gamma\left\|\mathbf{O}^{(t-1)}\right\| $
   % \STATE \textcolor{blue}{\# Block 4: Calculation perpendicular gradient component: }
   %  \STATE${\bbg^{(t)}}^\perp = \bbg^{(t)} - {\bbq^{(t)}} {\bbq^{(t)}}^\top \bbg^{(t)}$ \hfill{\color{gray}\COMMENT{$\bbo^{(t)}\perp{\bbg^{(t)}}^\perp$}}
   \STATE \textcolor{blue}{\# Block 4: Update weight in original space.}
    % \STATE $\bbw^{(t)} \leftarrow \bbw^{(t-1)}- \alpha \eta\left(\bbq^{(t)}\bbo^{(t)}+{\bbg^{(t)}}^\perp\right) -\eta \cdot \lambda \bbw^{(t-1)}$
    \STATE $\bbw^{(t)}
\;\leftarrow\;
\bbw^{(t-1)}
\;-\;\alpha\eta\,
\Bigl(\bbg^{(t)}-\bbq^{(t)}\bigl(\hat{\bbg}^{(t)}-\bbo^{(t)}\bigr)\Bigr)
\;-\;\eta\,\lambda\,\bbw^{(t-1)}$
   \STATE $t \gets t + 1$
   \UNTIL{convergence criteria met {\color{gray} (e.g. epoch number, gradient norm $\|\bbg^{(t)}\|\leq\xi$)}}
    \STATE \textbf{return} $\mathbf{W}^{(T)}$ 
 \end{algorithmic}
 \label{alg::SUMO}

\end{algorithm}
\allowdisplaybreaks
 \vspace{ -0.1 in}
Note that, for clarity, we can assume, without loss of generality, that \( m \geq n \). In the opposite scenario, the projection matrix would multiply the gradient from the right-hand side.

\begin{theorem}[Convergence of SUMO]\label{thm:Convergence_SUMO}
For a loss function $\mathcal{L},$ and given architecture $\Phi$, suppose that the compositions of $f\equiv\mathcal{L}\left(\Phi(\cdot)\right)$ is $\beta$-smooth non-convex function that is bounded by some $M\in\mathbb{R}_+$. Let $\bbg^{(t)}_j$ denote the gradient matrix w.r.t. the $j$-th reversible layer $\bbw^{(t)}_j,$ at time $t\in\mathbb{N}$, for all $j\in[L]$, and $T_\ell,\ell\in\mathbb{N}$ times are set by a convergence criterion (that is, $\|\hat{\bbg}_{\mathsf{T}_\ell}\|\leq\varsigma_\ell$). Then, there exist $\mathsf{C}\in\mathbb{R}_+$ and $N$ such that for all $\mathsf{T}_N>\frac{\mathsf{C}}{\varepsilon^2}$, and $\frac{1}{\mathsf{T}_N}\sum_{i=0}^{N-1}\sum_{t=\mathsf{T}_{i}}^{\mathsf{T}_{i+1}-1}\left\|{\bbg^{(t)}_j}\right\|_F^2 \leq \varepsilon
$. Namely, Algorithm \ref{alg::SUMO} 
achieves an $\varepsilon$-critical point,\footnote{Also known as $\varepsilon$-stationary, see, e.g., \citep{cossonLowRankGradientDescent2023}.} i.e., $\left\|\bbg^{(t)}_j\right\|_F^2\leq\varepsilon$, for some $t\in\mathbb{N}$, and any $j\in[L]$.
\end{theorem}

The proof of Theorem~\ref{alg::SUMO} can be found in Appendix~\ref{apd:main_results}. We emphasize that the convergence proof for Galore in \cite{zhao2024galore} addresses only optimization within fixed subspaces, ignoring dynamic updates. AdaRankGrad's proof \cite{refael2025adarankgrad} first established guarantees for the complete dynamic-subspace updates, yet both prior works simplified the inner steps as standard SGD. In contrast, SUMO's convergence proof explicitly considers the exact optimization steps without simplifications. 

To reduce memory consumption, Algorithm \ref{alg::SUMO} applies per-layer weight updates during backpropagation, following recent works such as \cite{lv2024adalomolo}. This contrasts with conventional optimizers, which store full gradients and update all weights afterward, potentially leading to inefficiencies. Details for post-hoc adapter extraction are discussed in Appendix~\ref{addditinal_info}.

\section{Experiments}\label{sec::Experiments}

\begin{figure}[t]
    \centering
    \includegraphics[width=0.5\textwidth]{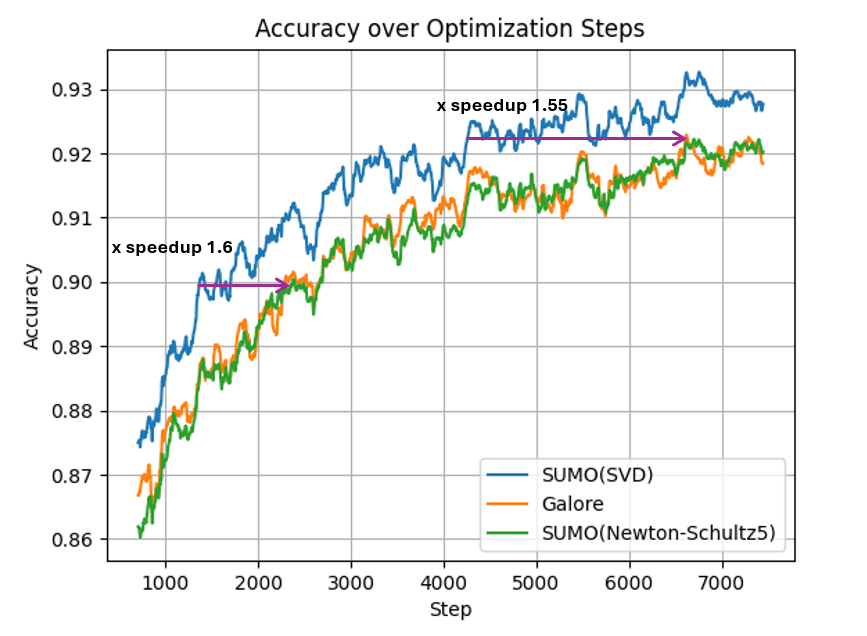}
    \caption{SUMO with SVD demonstrates superior convergence speed ($\sim\! 1.6\times$ faster), attaining comparable or higher accuracy than GaLore and SUMO with Newton-Schultz5 with significantly fewer optimization steps on QNLI.}
    \label{fig:anli_accuracy}
    \vskip -0.1 in
\end{figure}

\paragraph{Fine-tuning on GLUE benchmark.}
Our model was evaluated using the GLUE benchmark \citep{wang2019superglue} through the fine-tuning of the pre-trained Roberta-base model \cite{liu2019roberta} across eight tasks. The comparative analysis includes full fine-tuning, LoRA, and GaLore methodologies, with the results enumerated in Table \ref{tab:comparison}. The reported metrics are the overall accuracy (matched and mismatched) for MNLI, Matthew's correlation for CoLA, Pearson correlation for STS-B, F1-score for MRPC, and accuracy for the remaining tasks. Evidently, our approach improves fine-tuning accuracy while requiring less training memory, using only a single moment compared to GaLore. The experiments were carried out using the NVIDIA A100 GPU.

% \paragraph{Fine-tuning on GLUE benchmark.}
% Our model was evaluated using the GLUE benchmark \citep{wang2019superglue} through the fine-tuning of the pre-trained Roberta-base model \cite{liu2019roberta} across eight tasks. The comparative analysis includes full fine-tuning, LoRA, and GaLore methodologies, with the results enumerated in Table \ref{tab:comparison}. The metrics reported are the overall (matched and mismatched) accuracy for MNLI, Matthew's correlation for CoLA, Pearson correlation for STS-B, F1-score for MRPC, and accuracy for the remaining tasks. Evidently, our approach enhances the fine-tuning accuracy while requiring less training memory, utilizing only a single moment compared to GaLore. The experiments were carried out using the NVIDIA A100 GPU.
\begin{table}[hbt]
\centering
\caption{Comparison of SUMO against state-of-the-art memory-efficient fine-tuning methods on the GLUE benchmark using the pre-trained RoBERTa-Base model. For comparison, we provide detailed results for SUMO using both SVD and Newton-Schulz5 orthogonalizations (ablation study).}
\resizebox{\linewidth}{!}{
\begin{tabular}{lccccccccc}
\hline
\textbf{Model}  & \textbf{Memory} & \textbf{CoLA} & \textbf{STS-B} & \textbf{MRPC} & \textbf{RTE} & \textbf{SST2} & \textbf{MNLI} & \textbf{QNLI} & \textbf{QQP} \\
\hline
Full Fine-Tuning  & 747M & 62.24 & 90.92 & 91.30 & 79.42 & 94.57 & 87.18 & 92.33 & 92.28 \\ \hline
LoRA (rank=4) & 257M & 61.38 & 90.57 & 91.07 & 78.70 & 92.89 & 86.82 & 92.18 & 91.29 \\
GaLore (rank=4) & 253M & 60.35 & 90.73  & 92.25 & 79.42 & 94.0  & 87.0 & 92.24 & 91.06 \\ 
SUMO (Newton-Schulz5, rank=4) & \textbf{197M} & 61.8 & 90.82 & 92.43 & 79.36 & 94.17  & 86.92 & 92.26 & 91.27 \\
SUMO (SVD, rank=4) & \textbf{197M} & \textbf{62.3} & \textbf{91.04} & \textbf{93.5} & \textbf{81.07} & \textbf{94.93}  & \textbf{87.34} & \textbf{93.26} & \textbf{91.68} \\
\hline
LoRA (rank=8) & 264M & \textbf{61.83} & 90.80 & 91.90 & 79.06 & 93.46 & 86.94 & 92.25 & 91.22 \\
GaLore (rank=8) & 257M & 60.06 & 90.82 & 92.0 & 79.78 & 94.38 & 87.17 & 92.2   & 91.11  \\
SUMO (Newton-Schulz5, rank=4) & \textbf{198M} & 61.74 & 90.79 & 91.94 & 79.69 & 94.17  & 87.21 & 92.24 & 91.38 \\
SUMO (rank=8) & \textbf{198M} & 61.7 & \textbf{91.1} & \textbf{93.7} & \textbf{81.37} & \textbf{94.82}  & \textbf{87.58} & \textbf{93.67} & \textbf{91.72} \\
\hline
\end{tabular}}
\label{tab:comparison}
%\vskip -0.2 in
\end{table}
% \begin{lemma}[conv]
% \end{lemma}
% connection to ill-cond in conv proof
% condition number for reversibale layer 
% numerical rank
%%%%%%%%%%%%%%%%%%%%%%%%%%%%%%%%%%%%%%%%%%%%%%%%%%%%%%%%%%%%
\paragraph{Pre-training LLAMA on C4 Dataset.}
To highlight the effectiveness of our method in pre-training, we pre-train LLaMA models following the evaluation protocol of \cite{zhao2024galore}, comparing the performance to the state-of-the-art method, in terms of perplexity and memory usage. Specifically, we compare SUMO's performance with state-of-the-art methods on perplexity and memory efficiency. For this evaluation, we trained large LLaMA-based models on the C4 dataset, a curated and extensive version of the Common Crawl web corpus \citep{raffel2020exploring}. This dataset is widely used for pre-training language models and developing word representations. To better reflect real-world pre-training scenarios, we trained on a non-repeating, large-scale dataset and scaled model sizes up to 1 billion parameters. The results of these experiments are shown in Table~\ref{ex::pretraining}. Experiments were conducted using an NVIDIA H200 GPU.
\begin{table}[htb]
\caption{Comparison of state-of-the-art low-rank algorithms for pre-training LLaMA models of varying sizes on the C4 dataset. The results are reported in terms of validation perplexity. As shown, SUMO leads to improved performance with substantial memory reduction compared to leading parameter-efficient fine-tuning schemes. }
\centering
\begin{tabular}{l@{\hskip 4pt}c@{\hskip 4pt}c@{\hskip 4pt}c@{\hskip 4pt}c}
\hline 
\textbf{Method} & \textbf{60M} & \textbf{130M} & \textbf{350M} & \textbf{1B}  \\
\hline 
Full-Rank & $34.06$ ($0.36$G) & $25.08$ ($0.76$G) & $18.80$ ($2.06$G)&  $15.56 (7.80\mathrm{G})$\\
\hline 
GaLore    & $34.88$ ($0.24$G) & $25.36$ ($0.52$G) & $18.95$ ($1.22$G) &  $15.64 (4.38\mathrm{G})$\\
Low-Rank  & $78.18$ ($0.26$G) & $45.51$ ($0.54$G) & $37.41$ ($1.08$G) & $142.53 (3.57\mathrm{G})$\\
LoRA      & $34.99$ ($0.36$G) & $33.92$ ($0.80$G) & $25.58$ ($1.76$G) & $19.21 (6.17\mathrm{G})$\\
ReLoRA    & $37.04$ ($0.36$G) & $29.37$ ($0.80$G) & $29.08$ ($1.76$G) & $18.33 (6.17\mathrm{G})$\\
SUMO  & $\textbf{34.26}$ ($\mathbf{0.23}$G) & $\textbf{24.87}$ ($\mathbf{0.51}$G) & $\textbf{18.69}$ ($1.16$G) & $\textbf{14.68}$ ($3.84$G)\\
\hline
\makecell{Training Tokens}  & $1.1$B & $2.2$B & $6.4$B & 13.1B\\
$r / d_{\text{model}}$   & $128 / 256$ & $256 / 768$ & $256 / 1024$ & $512 / 2048$\\
\hline
\end{tabular}
\label{ex::pretraining}
% \vskip -0.1 in
\end{table}
\paragraph{Few/Zero-shot reasoning and long-context generalization.}
To evaluate the performance of our method on a complex reasoning task, we utilize the GSM8K dataset \cite{cobbe2021training} to test systematic generalization. For these experiments, we used a batch size of 32 and 10 epochs for fine-tuning. We present the performance result in Table \ref{phi_gs8k} training Phi-2 (2.7B) model \cite{javaheripi2023phi}, and in Table \ref{lama_gs8k} training LLaMA (3B) model \cite{touvron2023llama}. The results demonstrate that the proposed method significantly improves generalization to out-of-distribution data. We used NVIDIA H200 GPU.
% \begin{table}[htb]
% \centering
% \caption{Zero shot evaluation performance on GSM8K dataset, training Phi-2 (2.7B).}
% \begin{tabular}{lcc}
% \hline
% Phi-2 (2.7B)& Rank& Accuracy (0-shot) \\
% \hline Base Model& 64& $15.16 \%$ \\
%  Galore& 64& $52.24 \%$ \\
%  LoRA& 64& $42.8 \%$ \\
%  LORENZA& 64& $\textbf{53.37}$\%\\
%  \hline\label{phi_gs8k}
% \end{tabular}
% \end{table}

% \begin{table}[htb]
% \centering
% \caption{8-shot evaluation performance on GSM8K dataset, training  LLaMA (3B).}
% \begin{tabular}{lcc}
% \hline
% LLaMA (1B) & Rank & Accuracy (8-shot) \\
% \hline Base Model & 64& $17.93\%$ \\
%  Galore & 64 & $74.9 \%$ \\
%  LoRA & 64& $68.3 \%$ \\
%  LORENZA & 64& $\textbf{76.4}$\%\\
%  \hline\label{lama_gs8k}
% \end{tabular}
% \end{table}

\begin{table}[htb]
\centering
\begin{minipage}{0.48\textwidth}
\centering
\caption{Zero-shot evaluation on GSM8K dataset (Phi-2, 2.7B).}
\begin{tabular}{lcc}
\hline
Phi-2 (2.7B) & Rank & Accuracy (0-shot) \\
\hline
Base Model & 64 & $15.16\%$ \\
Galore     & 64 & $52.24\%$ \\
LoRA       & 64 & $42.8\%$ \\
SUMO    & 64 & $\textbf{54.13}\%$ \\
\hline
\end{tabular}
\label{phi_gs8k}
\end{minipage}
\hfill
\begin{minipage}{0.48\textwidth}
\centering
\caption{8-shot evaluation on GSM8K dataset (LLaMA, 3B).}
\begin{tabular}{lcc}
\hline
LLaMA (3B) & Rank & Accuracy (8-shot) \\
\hline
Base Model & 64 & $17.93\%$ \\
Galore     & 64 & $74.9\%$ \\
LoRA       & 64 & $68.3\%$ \\
SUMO    & 64 & $\textbf{76.7}\%$ \\
\hline
\end{tabular}
\label{lama_gs8k}
\end{minipage}
\end{table}
Additional experiments and ablation studies are presented in the Appendix~\ref{Additinal_Experiments}. 

\paragraph{Zero-shot generalization.}
Following this, we extend our evaluation to diverse commonsense and reasoning tasks using zero-shot methods, and then conduct the exact evaluation protocol described in Table 4 of \cite{zhu2024apollo} (which suggests the Apollo optimizer). The task details are presented in Appendix~\ref{d6}. To ensure a fair comparison, we repeat the exact experimental setup used for AdamW, APOLLO, and APOLLO-Mini. The results are summarized in Table~\ref{tab:zero_shot_reasoning} below.

\begin{table*}[htb]
\centering
\small
\caption{Zero-shot evaluation of LLaMA-350M models pretrained with sequence length 1024 across reasoning tasks (lower is better for perplexity; higher is better otherwise).}
\label{tab:zero_shot_reasoning}
\resizebox{\linewidth}{!}{
\begin{tabular}{lccccccccccccc} % 1 + 13 columns
\toprule
\textbf{Method} & \textbf{Memory} & \textbf{Perplexity} & \textbf{BoolQ} & \textbf{RTE} & \textbf{HS} & \textbf{WG} & \textbf{OBQA} & \textbf{ARC-E} & \textbf{ARC-C} & \textbf{PIQA} & \textbf{SciQ} & \textbf{MathQA} & \textbf{Avg.} \\
\midrule
AdamW        & 1.37G & 16.30 & 0.4917 & 0.4693 & 0.3688 & 0.5233 & 0.332 & 0.3729 & 0.2449 & 0.6534 & 0.609 & 0.2064 & 0.4272 \\
APOLLO       & 0.34G & 15.64 & 0.5373 & 0.4698 & 0.3850 & 0.4925 & 0.322 & 0.3788 & 0.2483 & 0.6681 & \textbf{0.624} & 0.2127 & 0.4406 \\
APOLLO-Mini  & 0.15G & 16.12 & 0.5376 & 0.4562 & 0.3707 & 0.5217 & \textbf{0.324} & 0.3758 & 0.2312 & 0.6638 & 0.619 & 0.2224 & 0.4374 \\
SUMO         & 0.18G & \textbf{15.49} & \textbf{0.5479} & \textbf{0.4709} & \textbf{0.3937} & \textbf{0.5313} & 0.321 & \textbf{0.3832} & \textbf{0.2496} & \textbf{0.6709} & 0.623 & \textbf{0.2246} & \textbf{0.4416} \\
\bottomrule
\end{tabular}}
\end{table*}

As shown, the SUMO-pretrained LLaMA-350M model achieves lower perplexity on average and consistently outperforms on downstream benchmarks.

\section{Discussion}
\vskip -0.1 in
Our results highlight that exact moment orthogonalization within a low-dimensional adaptive subspace significantly improves both convergence and stability in memory-efficient LLM training. By avoiding the approximation errors of Newton–Schulz5, the proposed SUMO leverages the low-rank structure of gradients to enable accurate, spectral-norm-aligned updates with minimal overhead. 

Empirically, SUMO outperforms prior low-rank methods on both fine-tuning and pre-training tasks, achieving greater memory reduction than memory-efficient benchmarks such as Galore. Our theoretical analysis further confirms its superior convergence properties under practical conditions. These findings position SUMO as a simple yet effective alternative to approximate geometric optimizers. Future work may investigate parallel computations for orthogonalization, integrate quantization techniques, and assess the effectiveness of the method in knowledge editing \cite{rozner2024knowledge} or domain generalization \cite{refael2025lorenza,roznerdomain}.

\section*{Acknowledgment}
TT was supported by the Israel Science Foundation (No. 1940/23) and MOST (No. 0007091) grants. OL was supported by the MOST grant No. 0007341.

\bibliography{bib}

%%%%%%%%%%%%%%%%%%%%%%%%%%%%%%%%%%%%%%%%%%%%%%%%%%%%%%%%%%%%

\newpage

\newpage
\appendix
\section{Proofs of Section~\ref{sec:main_results}}\label{apd:main_results}
In this section, we prove all the theorems and results of Section~\ref{sec:main_results}. 

\begin{proof}[Proof of Lemma \ref{lem::moment_lowrank}]

We aim to show that if the gradient $\bbg^{(t)}$ becomes approximately rank-one exponentially fast, then the exponentially weighted moving average of the gradients (i.e., the momentum $\bbm^{(t)}$) also exhibits exponential decay of higher-rank components.

Consider the singular value decomposition of the gradient $\bbg^{(t)}= \bbu^{(t)}\Sigma^{(t)}{\bbv^{(t)}}^\top,$ at iteration $t$. For all natural numbers $r<m,$, we define ${\bbh^{(t)}}^{m\times r}(r)=\bbu[:,1:r]$. To enhance notation clarity, denote $\bbp^{(t)}(r)=\bbh^{(t)}(r) {\bbh^{(t)}}^\top(r),$, where $\bbp^{(t)}(r)$ represents an orthogonal projection matrix, satisfying the conditions ${\bbp^{(t)}}^\top(r)\bbp^{(t)}(r) = \bbp^{(t)}(r),$ and $\bbp^{(t)}(r) = {\bbp^{(t)}}^\top(r)$. Without compromising generality, it is assumed that at $t=0$, the rank of $\bbg^{(0)}$ is characterized by $\text{rank}\left(\bbg^{(0)}\right)>r$. For reversible networks, it has been established in \citep{zhao2024galore}[Theorem 3.2] that the gradients assume the form $\bbg^{(t)}=\frac{1}{N} \sum_{i=1}^N\left(\bba_i-\bbb_i \bbw^{(t)} \bbc_i\right)$, characterized by constant matrices $\{\bba_i\}_i$ and positive semi-definite (PSD) matrices $\{\bbb_i,\bbc_i\}_i$, for $t \geq \mathsf{t}_0$, where $\mathsf{t}_0\in\mathbb{N}$ holds. It is pertinent to recall that the vanilla weight update can be represented as $\bbw^{(t)}=\bbw^{(t-1)}+\eta \bbg^{(t-1)}$. Let $\bbs\triangleq\frac{1}{N} \sum_{i=1}^N \bbc_i \otimes \bbb_i$ and $\lambda_1<\lambda_2$ denote its two smallest distinct eigenvalues. To substantiate our findings, we utilize several results and arguments presented in the proof of Lemma 3.3 in \citep{zhao2024galore}. Specifically, consider $\bbg^{(\mathsf{t}_0)}$ as the projection of $\bbg^{(\mathsf{t}_0)}$ onto the minimal eigenspace $\mathcal{V}_1$ of $S$ corresponding to $\lambda_1.$. According to our assumption, the rank of $\bbg^{(\mathsf{t}_0)}$ is $L$, and its singular value decomposition (SVD) is given by $\bbg^{(\mathsf{t}_0)}=\sum_{l=1}^L c_l \boldsymbol{z}_l \boldsymbol{y}_l^{\top}$, where $\left\{\boldsymbol{z}_l\right\}_{l=1}^L$ and $\left\{\boldsymbol{y}_l\right\}_{l=1}^L$ are orthonormal unit vectors, and $\left\{c_l\right\}_{l=1}^L $ are the corresponding singular values. Therefore, as per Lemma 3.3 in \citep{zhao2024galore}, the gradient can be decomposed into,
    $$\left\|\bbg^{(t)}\right\|_F^2 \leq\left(1-\eta \lambda_2\right)^{2 t}\left\|g_0^{\perp}\right\|_2^2+\left(1-\eta \lambda_1\right)^{2 t}\left\|g_0^{\|}\right\|_2^2,$$ 
  
where $g_0^{\parallel}$ is the projection of $\bbg^{(0)}$ onto the minimal eigenspace $\mathcal{V}_1$ of $S = \frac{1}{N} \sum_{i=1}^N \bbc_i \otimes \bbb_i$, and $g_0^{\perp}$ is orthogonal to $\mathcal{V}_1$. Here, $\lambda_1 < \lambda_2$ are the smallest distinct eigenvalues of $S$.

We now unroll the momentum update, $\bbm^{(t)} = \sum_{s=1}^{t} \beta^{t-s} \bbg^{(s)}.$ Substitute the decomposition of $\bbg^{(s)}$,
$$
 \left\|\bbm^{(t)}\right\|_F^2 \leq \sum_{s=1}^{t} \beta^{t-s} \left[ (1 - \eta \lambda_1)^s g_0^{\parallel} + (1 - \eta S)^s g_0^{\perp} \right] 
=  \sum_{s=1}^{t} \beta^{t-s}(1 - \eta \lambda_1)^s g_0^{\parallel} + \sum_{s=1}^{t} \beta^{t-s}(1 - \eta S)^s g_0^{\perp}.
$$

Let us define $a_t \triangleq \sum_{s=1}^{t} \beta^{t-s}(1 - \eta \lambda_1)^s, \quad b_t \triangleq \sum_{s=1}^{t} \beta^{t-s}(1 - \eta S)^s g_0^{\perp},$ so that $\left\|\bbm^{(t)}\right\|_F^2 = a_t g_0^{\parallel} + b_t$. Now, compute the squared Frobenius norm:

$$
\|\bbm^{(t)}\|_F^2 = \|a_t g_0^{\parallel} + b_t\|_F^2 = a_t^2 \|g_0^{\parallel}\|_F^2 + 2 a_t \langle g_0^{\parallel}, b_t \rangle + \|b_t\|_F^2.
$$

Since $g_0^{\parallel} \perp g_0^{\perp}$ and $b_t$ lies in the span of $g_0^{\perp}$, we have $\langle g_0^{\parallel}, b_t \rangle = 0$, thus,

$$
\|\bbm^{(t)}\|_F^2 = a_t^2 \|g_0^{\parallel}\|_F^2 + \|b_t\|_F^2.
$$

Likewise, the spectral norm $\|\bbm^{(t)}\|_2^2 \ge a_t^2 \|g_0^{\parallel}\|_2^2$. Hence, the ratio

$$
\kappa_m(t) = \frac{\|\bbm^{(t)} - \bbp^{(t)}(1)\bbm^{(t)}\|_F^2}{\|\bbm^{(t)}\|_F^2} 
\le \frac{\|\bbm^{(t)}\|_F^2 - \|\bbm^{(t)}\|_2^2}{\|\bbm^{(t)}\|_F^2} 
\le \frac{a_t^2 \|g_0^{\parallel}\|_F^2 + \|b_t\|_F^2 - a_t^2 \|g_0^{\parallel}\|_2^2}{a_t^2 \|g_0^{\parallel}\|_F^2 + \|b_t\|_F^2}.
$$

Using that $\|g_0^{\parallel}\|_2^2 = \sigma_1^2$, and the decay bound $\|b_t\|_F^2 = O((\max\{\beta, 1 - \eta \lambda_2\})^{2t})$, while $a_t^2 = \Omega((\max\{\beta, 1 - \eta \lambda_1\})^{2t})$, we conclude:

$$
\kappa_m(t) \le O\left(\left(\frac{\max\{\beta, 1 - \eta \lambda_2\}}{\max\{\beta, 1 - \eta \lambda_1\}}\right)^{2t}\right) = O(C^{-t}),
$$

for some constant $C>1$ .
\end{proof}

Before proving Lemma~\ref{lma:moun_convergence}, we shortly present the following two preliminary lemmas. To that end, we present the following notations,
\begin{itemize}
    \item $\bbm^{(t)}$ - The moment in iteration $t$. Its dimensions are $n \times m$, where $n < m$.
    \item $\| \cdot \|$ - The Frobenius norm: $\|\bba\| = \|\bba\|_F = \sqrt{\bba \bba^\top}$
    \item $\mathcal{L}^*$ - A stationary point to which the loss $\mathcal{L}$ converges.
    \item $B$ - Batch size.
    \item For $\bba \in \real^{m \times m}$ and $\bbb \in \real^{n \times n}$ we denote $\bba \bbi_{m \times n} \bbb^\top$ by $\bba \bbb^\top$ for convenience.
\end{itemize}

Additionally, we note that our proof is based on an equivalent but slightly modified formulation of moment's update. Specifically, instead of using the standard formulation of the moment's update 
\[
\bbm^{(t+1)} = \beta \bbm^{(t)} + \bbg^{(t)},
\]
we consider the convex combination,
\[
\bbm^{(t+1)} = \beta \bbm^{(t)} + (1 - \beta)\bbg^{(t)}.
\]
This alternative formulation simplifies the analysis, but equivalent. To show that, we point out that we can choose an modefied learning step \(\eta^* = \frac{\eta}{1-\beta}>0\) we get the same weight's updating step.
\begin{align*}
    &\eta^* Orth \left(\beta \bbm^{(t)} + (1-\beta)\bbg^{(t)} \right) \\
    &= Orth \left(\eta^* \beta \bbm^{(t)} + \eta^* (1-\beta) \bbg^{(t)} \right) \\
    &= Orth \left(\eta \frac{\beta}{1-\beta} \bbm^{(t)} + \eta \bbg^{(t)} \right) \\
    &= \eta Orth \left(\frac{\beta}{1-\beta} \bbm^{(t)} + \bbg^{(t)} \right),
\end{align*}
where \(Orth\) is the SVD orthogonalization step, formaly solving $$\underset{\bbo}{\operatorname{arg\,min}} \{ \|\bbo - \bba\|_F : \text{either } \bbo^T \bbo = \bbi \text{ or } \bbo\bbo^T = \bbi \}.$$ Obviously, $\beta>0$ could be chosen in a way that $\frac{\beta}{1-\beta} \bbm^{(t)}$ would result in any required positive real number.

We assume the following 4 assumptions throughout our proofs:
\begin{itemize}
    \item[(A1)] The gradient \(\nabla \mathcal{L}(\mathbf{W})\) is \(L\)-Lipschitz continuous.\label{A1}
    \item[(A2)] \(\nabla \mathcal{L}(\bbw, \xi) \) is an unbiased estimator of \(\nabla \mathcal{L}(\bbw)\)  where \(\mathcal{L}(\bbw, \xi)\) is the gradient of \(\mathcal{L}(\bbw)\) when taking a single training sample $\xi$.\label{A2}
    \item[(A3)] \(\mathbb{E}\|\nabla \mathcal{L}(\bbw, \xi) - \nabla \mathcal{L}(\bbw) \| \leq \sigma^2 \).
    \item[(A4)] There exists \(\delta > 0\) such that \(\|\mathbf{\mathcal{E}}^{(t)}_5\| \leq \delta \|\bbu^{(t)}{\bbv^{(t)}}^\top\| = \delta \sqrt{m}\) for all \(t\).
\end{itemize}  

\begin{lemma}[Descent Lemma with Newton-Schulz Approximation Error]\label{lemma:first}
Consider the Muon optimizer update defined by
\begin{align*}
    & \mathbf{M}^{(t)} \leftarrow \beta \mathbf{M}^{(t-1)}+(1-\beta) \bbg^{(t)}, \\
    & \bbo^{(t)} \leftarrow \bbu^{(t)}{\bbv^{(t)}}^{\top}+\mathbf{\mathcal{E}}^{(t)}_5, \quad \text{(Newton-Schulz 5 iteration approximation)}, \\
    & \mathbf{W}^{(t+1)} \leftarrow \mathbf{W}^{(t)} - \eta_t \bbo^{(t)},
\end{align*}
where \(\mathbf{M}^{(t)} = \bbu^{(t)} \bbs^{(t)} {\bbv^{(t)}}^{\top}\) is the singular value decomposition of \(\mathbf{M}^{(t)}\), and \(\mathbf{\mathcal{E}}^{(t)}_5\) represents the Newton-Schulz (5 iterations) approximation error. Additionally, assume (A1) - (A4).
Then the following holds:
\footnotesize
\begin{align*}
    &\mathcal{L}(\mathbf{W}^{(t+1)}) \leq  \\
    &\mathcal{L}(\mathbf{W}^{(t)}) - \left( \frac{\eta_t}{4} - \eta_t \sqrt{m} \delta \right) \| \nabla \mathcal{L}(\mathbf{W}^{(t)}) \| + \eta_t \frac{5}{2} \| \nabla \mathcal{L}(\mathbf{W}^{(t)}) - \mathbf{M}^{(t)} \| + \frac{\eta_t^2 m L}{2} +\eta_t^2 m L \delta + \frac{\eta_t^2 L m \delta^2}{2}
\end{align*}
\end{lemma}

\begin{proof}
    Since $\mathcal{L}$ is L-lipschitz function, the descent lemma holds. Thus we have
    \footnotesize
    \begin{align*}
        \mathcal{L}(\mathbf{W}^{(t+1)}) &\leq \mathcal{L}(\mathbf{W}^{(t)}) + \langle\nabla \mathcal{L}(\mathbf{W}^{(t)}), \mathbf{W}^{(t+1)} - \mathbf{W}^{(t)} \rangle + \frac{L}{2} \|\mathbf{W}^{(t+1)} - \mathbf{W}^{(t)} \| ^2 \nonumber \\
         &= \mathcal{L}(\mathbf{W}^{(t)}) - \eta_t \langle \nabla\mathcal{L}(\mathbf{W}^{(t)}), \bbu^{(t)}{\bbv^{(t)}}^\top + \mathbf{\mathcal{E}}^{(t)}_5 \rangle + \frac{L \eta_t^2}{2} \|\bbu^{(t)} {\bbv^{(t)}}^\top + \mathbf{\mathcal{E}}^{(t)}_5 \|^2 \nonumber \\
         &= \mathcal{L}(\mathbf{W}^{(t)}) - \eta_t \langle \nabla \mathcal{L}(\mathbf{W}^{(t)}), \bbu^{(t)} {\bbv^{(t)}}^\top \rangle -\eta_t \langle \nabla \mathcal{L}(\mathbf{W}^{(t)}), \mathbf{\mathcal{E}}^{(t)}_5 \rangle \\
         &+\frac{L \eta_t^2}{2} \left( n + 2 \langle \bbu^{(t)} {\bbv^{(t)}}^\top, \mathbf{\mathcal{E}}^{(t)}_5 \rangle + \|\mathbf{\mathcal{E}}^{(t)}_5\|^2 \right) \nonumber \\
         &\underset{(*)} {\leq}\mathcal{L}(\mathbf{W}^{(t)}) - \frac{\eta_t}{4} \| \nabla \mathcal{L}(\mathbf{W}^{(t)}) \| + \eta_t \frac{5}{2}\|\nabla \mathcal{L}(\mathbf{W}^{(t)}) - \mathbf{M}^{(t)}\| + \frac{\eta_t^2 m L}{2} \nonumber \\
         &+ \eta_t \| \nabla \mathcal{L}(\mathbf{W^{(t)}})\| \|\mathbf{\mathcal{E}}^{(t)}_5\| + L \eta_t^2 \sqrt{m} \|\mathbf{\mathcal{E}}^{(t)}_5\| + \frac{L\eta_t^2}{2}\|\mathbf{\mathcal{E}}^{(t)}_5\|^2 \nonumber \\
         &\leq \mathcal{L}(\mathbf{W}^{(t)}) - \frac{\eta_t}{4} \| \nabla \mathcal{L}(\mathbf{W}^{(t)}) \| + \eta_t \frac{5}{2} \|\nabla \mathcal{L}(\mathbf{W}^{(t)}) - \mathbf{M}^{(t)}\| + \frac{\eta_t^2 m L}{2} \\ \label{with_no_error}
         &+ \eta_t \delta \sqrt{m}\| \nabla \mathcal{L}(\mathbf{W}^{(t)}) \| + L \eta_t^2 \sqrt{m} \delta \sqrt{m} + \frac{L \eta_t^2}{2} \delta^2 n \nonumber \\
         &= \mathcal{L}(\mathbf{W}^{(t)}) - \left( \frac{\eta_t}{4} - \eta_t \sqrt{m} \delta \right) \| \nabla \mathcal{L}(\mathbf{W}^{(t)}) \| + \eta_t \frac{5}{2} \| \nabla \mathcal{L}(\mathbf{W}^{(t)}) - \mathbf{M}^{(t)} \| + \frac{\eta_t^2 m L}{2} + \nonumber \\
         &+ \eta_t^2 m L \delta + \frac{\eta_t^2 L m \delta^2}{2} \nonumber
    \end{align*}
    Where in $(*)$ we used \citep{cutkosky2020momentumimprovesnormalizedsgd}, equation 2.8.
\end{proof}

\begin{lemma}\label{lma:second}
    For constant $\eta_t = \eta>0$, the following holds
    \begin{align}
        &\frac{\eta - 4 \eta \sqrt{m} \delta}{4}\sum_{t = 1}^T \| \nabla \mathcal{L}(\mathbf{W}^{(t)}) \| \leq  \nonumber \\
        &\mathcal{L}(\mathbf{W}^{(1)}) - \mathcal{L}^* + \eta_t \frac{5}{2} \sum_{t=1}^T \| \nabla \mathcal{L}(\mathbf{W}^{(t)}) - \mathbf{M}^{(t)} \| + \frac{\eta^2 m L T}{2} + T\eta^2 m L \delta + \frac{T\eta^2 L m \delta^2}{2} \nonumber.
    \end{align}
\end{lemma}

\begin{proof}
    Using Lemma~\ref{lemma:first}, isolating $\| \nabla \mathcal{L}(\mathbf{W}^{(t)}) \|$ and summing over all steps
    \begin{align}
        &\frac{\eta - 4 \eta \sqrt{m} \delta}{4}\sum_{t = 1}^T \| \nabla \mathcal{L}(\mathbf{W}^{(t)}) \| \leq  \nonumber \\
        &\mathcal{L}(\mathbf{W}^{(1)}) - \mathcal{L}^* + \eta_t \frac{5}{2} \sum_{t=1}^T \| \nabla \mathcal{L}(\mathbf{W}^{(t)}) - \mathbf{M}^{(t)} \| + \frac{\eta^2 m L T}{2} + T\eta^2 m L \delta + \frac{T\eta^2 L m \delta^2}{2} \nonumber
    \end{align}
\end{proof}

\begin{proof}[Proof of Lemma~\ref{lma:moun_convergence}]
The proof follows \citep{cutkosky2020momentumimprovesnormalizedsgd}. Using the same notations as \citep{cutkosky2020momentumimprovesnormalizedsgd}, we denote $\hat{\gamma}^{(t)} = \mathbf{M}^{(t)} - \nabla \mathcal{L}(\mathbf{W}^{(t)})$, $\gamma^{(t)} = \bbg^{(t)} - \nabla \mathcal{L}(\mathbf{W}^{(t)})$ and $S(\bbx,\bby) = \nabla \mathcal{L}(\bbx) - \nabla \mathcal{L}(\bby)$. Note that we have the following
\begin{itemize}
    \item \(\mathbb{E}[\gamma^{(t)}] = 0\) from A(2).
    \item \(\mathbb{E}[\|\gamma^{(t)}\|^2] \leq \frac{\sigma^2}{m}\) from A(3).
    \item \(\mathbb{E}[\langle \gamma^{(i)}, \gamma^{(j)} \rangle] = 0,\ \forall i \neq j\) since \(\gamma^{(i)}\) and \(\gamma^{(j)}\) are independent.
    \item \(\|S(\bbx,\bby)\| \leq L \|\bbx - \bby\|\) from A(1).
\end{itemize}
% \begin{align}
%     \mathbb{E}[\gamma^{(t)}] = 0, \quad \mathbb{E}[\|\gamma^{(t)}\|^2] \leq \frac{\sigma^2}{m}, \nonumber \\
%     \mathbb{E}[\langle \gamma^{(i)}, \gamma^{(j)} \rangle] = 0,\ \forall i \neq j \nonumber \\
%     \|S(\bbx,\bby)\| \leq L \|\bbx - \bby\| \tag{2.10} \nonumber
% \end{align}

Now following the update in (2), we get
\begin{align*}
    \hat{\gamma}^{(t+1)} &= \beta \hat{\gamma}^{(t)} + (1 - \beta)\gamma^{(t)} + S(\bbx^{(t)}, \bbx^{(t+1)}) \\
    &= \beta^t \hat{\gamma}^{(1)} + (1 - \beta)\sum_{\tau=0}^{t-1} \beta^\tau \gamma^{(t - \tau)} + \sum_{\tau=0}^{t-1} \beta^\tau S(\bbx^{(t - \tau)}, \bbx^{(t + 1 - \tau)}),
\end{align*}
therefore
\[
\|\hat{\gamma}^{(t+1)}\| \leq \beta^t \|\hat{\gamma}^{(1)}\| + (1 - \beta)\left\| \sum_{\tau=0}^{t-1} \beta^\tau \gamma^{(t - \tau)} \right\| + \eta L \sum_{\tau=0}^{t-1} \beta^\tau.
\]

Taking expectation we get (using the fact that $\hat{\delta}_1 = \delta_1$):
\begin{align*}
    \mathbb{E}\|\hat{\gamma}^{(t+1)}\| &\leq \beta^t \frac{\sigma}{m} + (1 - \beta) \sqrt{ \sum_{\tau=0}^{t-1} \beta^{2\tau} \frac{\sigma^2}{B} } + \eta L \sum_{\tau=0}^{t-1} \beta^\tau \\
    &\leq \frac{\sigma}{m} \beta^t + \frac{\sigma}{m} \frac{1 - \beta}{\sqrt{1 - \beta^2}} + \eta L \frac{1}{1 - \beta} \\
    &\leq \frac{\sigma}{m} \beta^t + \frac{\sigma}{m} \sqrt{1 - \beta} + \eta L \frac{1}{1 - \beta}.
\end{align*}

All in all, we get
$$\sum_{t=1}^T \mathbb{E} [\|\hat{\gamma}^{(t+1)}\|] \leq \frac{\sigma}{(1-\beta)B} + T\sqrt{1-\beta}\frac{\sigma}{m} + \frac{T \eta L}{1- \beta}.$$

Using Lemma~\ref{lma:second}, we get
\begin{align}
    &\frac{\eta_t - 4L\eta_t \sqrt{m} \delta}{4}\sum_{t = 1}^T \| \nabla \mathcal{L}(\mathbf{W}^{(t)}) \| \leq  \nonumber \\
    &\mathcal{L}(\mathbf{W}^{(1)}) - \mathcal{L}^* + \eta_t \frac{5}{2} \sum_{t=1}^T \| \nabla \mathcal{L}(\mathbf{W}^{(t)}) - \mathbf{M}^{(t)} \| + \frac{\eta^2 m L T}{2} + T\eta^2 m L \delta + \frac{T\eta^2 L m \delta^2}{2} \nonumber.
\end{align}

Dividing both sides by $\frac{\eta - 4\eta L \sqrt{m} \delta}{4}$ we get
\footnotesize{
\begin{align}
        &\sum_{t = 1}^T \| \nabla \mathcal{L}(\mathbf{W}^{(t)}) \| \leq  \nonumber \\
        &\left[\frac{4(\mathcal{L}(\mathbf{W}^{(1)}) - \mathcal{L}^*)}{\eta} + 10 \sum_{t=1}^T \| \nabla \mathcal{L}(\mathbf{W}^{(t)}) - \mathbf{M}^{(t)} \| + 2\eta m L T  + 4T\eta m L \delta + 2T\eta L m \delta^2 \right] \cdot \frac{1}{1- 4\sqrt{m} \delta} \nonumber \\
        &\leq \left[ \frac{4R}{\eta} + 10 \frac{\sigma}{(1-\beta)m} + 10T \sqrt{1-\beta}\frac{\sigma}{m} + 10\frac{T \eta L}{1 - \beta} + 2 \eta m L T + 4T\eta m L \delta + 2T\eta L m \delta^2\right] \cdot \frac{1}{1- 4\sqrt{m} \delta} \nonumber
\end{align}}
    By taking $\eta = \sqrt{\frac{4R}{(10/(1-\beta) + 2m + 4m\delta + 2m\delta^2)TL}}$ we get
    \footnotesize{
    \begin{align*}
        &\sum_{t=1}^{T} \|\nabla \mathcal{L}(\mathbf{W}^{(t)})\| \\
        &\left[\leq 4 \sqrt{RTL(10 / (1-\beta) + 2m + 4m\delta + 2m\delta^2)} + \frac{10 \sigma}{(1-\beta) m} + 10T\sqrt{1-\beta} \frac{\sigma}{(1-\beta)m} \right] \frac{1}{1-4\sqrt{m} \delta}.
    \end{align*}}

    Now we have two types of parameter choice. If we take $B = 1$ (batch size free), we need to take 
    $1-\beta = \min(1, \frac{\sqrt{RL}}{\sigma\sqrt{T}})$ so that we have

    \footnotesize{
    \begin{align*}
      &\sum_{t=1}^{T} \|\nabla \mathcal{L}(\mathbf{W}^{(t)})\| \\
      &\leq \left[2\sqrt{RTL(2m + 4m\delta + 2m\delta^2)} + 2\sqrt{10 \cdot \sigma} \cdot (RL)^{1/4}T^{3/4} + 10\sigma^2\sqrt{\frac{T}{RL}} + 10\sqrt{\sigma}(RL)^{1/4}T^{3/4} \right] \frac{1}{1 - 4\sqrt{m} \delta} , 
    \end{align*}}    
    thus    
    $$\frac{1}{T}\sum_{t=1}^{T} E\left[\|\nabla \mathcal{L}(\mathbf{W}^{(t)})\| \right] \leq \mathcal{O} \left( \left[\frac{\sqrt{RLn(2 + 4\delta + 2\delta^2)}}{\sqrt{T}} + \sigma \cdot \frac{(RL)^{1/4}}{T^{1/4}} + \frac{\sigma^2}{\sqrt{RLT}} + \frac{\sqrt{\sigma}(RL)^{1/4}}{T^{1/4}} \right] \frac{1}{1 - 4\sqrt{m} \delta}\right).$$

    If we take $\beta$ as an arbitrary constant in $(0,1)$, then we will need to take $B = T$, so that

    $$\frac{1}{T}\sum_{t=1}^{T} \|\nabla \mathcal{L}(\mathbf{W}^{(t)})\| \leq \mathcal{O}\left(\left[\frac{\sqrt{RLn(2 + 4\delta + 2\delta^2)}}{\sqrt{T}} + \frac{\sqrt{RL}}{\sqrt{T}} + \frac{\sigma}{T^{3/2}} + \frac{\sigma}{\sqrt{T}}\right]\frac{1}{1 - 4\sqrt{m} \delta}\right).$$
\end{proof}  

\begin{proof}[Proof of Lemma \ref{lma:estimating_error}]
   We denote $\bbb = \bba \bba^\top$ , $\bbx_k$ the result after $k$ Newton-Schultz iterations, $\bbx_0 = \frac{\bbb}{\|\bbb\|_2}$ and $\bbo = \bbu \bbv^\top$ where $\bbb = \bbu \boldsymbol{\boldsymbol{\Sigma}} \bbv^\top$ is the SVD decomposition with $\sigma_1 \geq \sigma_2 \geq \dots \geq \sigma_m$ the singular values. 
   
   It is known that Newton-Schultz converges quadratically, so we have 
    $$ \|\mathbf{\mathcal{E}}_k\|_2 = \|\bbx_k - \bbo\|_2\leq \|\bbx_0 - \bbo\|_2^{2^k}$$
    We now bound $\|\bbx_0 - \bbo\|_2$. We know that $\bbx_0 = \frac{\bbu \boldsymbol{\Sigma} \bbv^\top}{\|\bbb\|_2} = \frac{\bbu \boldsymbol{\Sigma} \bbv^\top}{\sigma_1}$
    $$
        \|\bbx_0 - \bbo\|_2 = \left\|\frac{\bbu \boldsymbol{\Sigma} \bbv^\top}{\sigma_1} - \bbu\bbv^\top \right\|_2 = \left\|\bbu (\frac{\boldsymbol{\Sigma}}{\sigma_1} - \boldsymbol{I})\bbv^\top \right\|_2 = \left \| \frac{\boldsymbol{\Sigma}}{\sigma_1} - \boldsymbol{I} \right\|_2
    $$
    Where last equality is due to the fact that $\| \cdot \|_2$ is unitary invariant.
    The matrix $\frac{\boldsymbol{\Sigma}}{\sigma_1} - \boldsymbol{I}$ is diagonal with values $\frac{\sigma_i}{\sigma_1} - 1$ on the diagonal. From that observation we get that 
    $$\left \| \frac{\boldsymbol{\Sigma}}{\sigma_1} - \boldsymbol{I} \right \|_2 = \max_{i} \left|\frac{\sigma_i}{\sigma_1} - 1\right| = \max_{i} \left(1 - \frac{\sigma_i}{\sigma_1}\right) = 1 - \frac{\sigma_m}{\sigma_1} = 1 - \frac{1}{\kappa}$$
    For the Frobenius norm, we get a similar analysis. $\|\bbx_0 - \bbo\|_F = \left\|\frac{\boldsymbol{\Sigma}}{\sigma_1} - \boldsymbol{I} \right\|_F$ since $\| \cdot \|_F$ is unitary invariant, so we just need to calculate $\left\|\frac{\boldsymbol{\Sigma}}{\sigma_1} - \boldsymbol{I} \right\|_F$. It is known that $\left\|\frac{\boldsymbol{\Sigma}}{\sigma_1} - \boldsymbol{I} \right\|_F \leq \sqrt{r} \left\|\frac{\boldsymbol{\Sigma}}{\sigma_1} - \boldsymbol{I} \right\|_2$ so all in all we have 
    $$\left\|\frac{\boldsymbol{\Sigma}}{\sigma_1} - \boldsymbol{I} \right\|_F \leq \sqrt{r} \left\|\frac{\boldsymbol{\Sigma}}{\sigma_1} - \boldsymbol{I} \right\|_2 = \sqrt{r} \left(1 - \frac{1}{\kappa} \right).$$
\end{proof}

\begin{proof}[Proof of Theorem \ref{thm:Convergence_SUMO}]
 for any layer $j\in[L]$; in the following, for simplicity of notation, we ignore the index $j$ and use $\bbg^{(t)}$ instead. By Lemma \ref{lem::ineer_convergance}, the low-rank optimization block 1 in Algorithm \ref{alg::SUMO} is guaranteed to converge; we denote by $\mathsf{T}_\ell\in\mathbb{N}$ the time index $t$ at which we enter block 1 for the $\ell$th time (i.e., $\|\hat{\bbg}_j^{\left(\mathsf{T}_\ell\right)}\|\leq\varsigma_2$), for $\ell\in\mathbb{N}$. Furthermore, we recall that  $\bbg_j^{(t)}\triangleq\nabla_{\boldsymbol{\bbw_j} }f\left(\boldsymbol{\theta}^{(t)}\right)$; when clear from the context, we omit $j$ from ${\bbw_j}$, and use instead $\nabla_{\boldsymbol{\bbw_j} }f\left(\boldsymbol{\theta}^{(t)}\right)=\nabla f\left(\bbw^{(t)}\right)$.
    Consider the SVD decomposition of the gradient $\nabla_{\boldsymbol{\bbw^j} }f\left(\boldsymbol{\theta}_{\mathsf{T}_i}\right)= \bbu^{(\mathsf{T}_i)}\boldsymbol{\Sigma}^{(\mathsf{T}_i)}{\bbv^{(\mathsf{T}_i)}}^\top$. For $t\in[\mathsf{T}_i,\mathsf{T}_{i+1}-1]$, we define the projected gradient as $\hat{\bbg}^{(t)} \triangleq \bbp^{(\mathsf{T}_i)}(r) \bbg^{(t)},$ where $\bbp^{(\mathsf{T}_i)}(r)= \bbu^{(\mathsf{T}_i)}\left[:,:r\right]^\top$, using the exact truncated-SVD calculation (in Block 1). For simplicity, we refer to $\bbq^{(\mathsf{T}_i)}(r)$ as $\bbq^{(\mathsf{T}_i)},$ and we denote $\bbp^{(\mathsf{T}_i)}={\bbq^{(\mathsf{T}_i)}}^\top\bbq^{(\mathsf{T}_i)}$.  Next, let $h_t\triangleq f(\bbw^{(t)})-f(\bbw^{(\mathsf{T}_{i+1})})$, and $\eta_t$ denote the learning rate. Then, 
    % \label{eqn:upp}
% \end{align}
\begin{align}
h_{t+1}
& = f(\mathbf{W}^{(t+1)}) - f(\mathbf{W}^{(\mathsf{T}_{i+1})})  \nonumber\\
& = f\left(\mathbf{W}^{(t)} - \eta_t \left(\mathbf{W}^{(t+1)}-\mathbf{W}^{(t)}\right)\right) - f\left(\mathbf{W}^{(\mathsf{T}_{i+1})}\right)  \nonumber\\
& \leq f(\mathbf{W}^{(t)}) - f\left(\mathbf{W}^{(\mathsf{T}_{i+1})}\right)- \eta_t \langle \bbg^{(t)}, \bbo^{(t)} \rangle + \frac{\beta \eta_t^2}{2} \|\bbo^{(t)}\|_F^2  \nonumber\\
& \underset{(1)}{\leq} f(\mathbf{W}^{(t)}) - f\left(\mathbf{W}^{(\mathsf{T}_{i+1})}\right)- \frac{\eta_t}{4} \|\hat{\bbg}^{(t)}\|_F^2 + \frac{5}{2} \cdot \frac{\eta_t \sigma}{\sqrt{m}} + \frac{\beta \eta_t^2 n}{2} \nonumber\\
& \underset{(2)}{\leq} h_t  - \frac{\eta_t}{4} \|\hat{\bbg}^{(t)}\|_F^2 + \frac{5}{2} \cdot \frac{\eta_t \sigma}{\sqrt{m}} + \frac{\beta \eta_t^2 n}{2} \nonumber  .
\label{eqn:descent}
\end{align}

where $(1)$ follows Eq \ref{r1}, and $(2)$ follows Lemma \ref{lem::ineer_convergance}. Thus, summing over $t \in [\mathsf{T}_i, \mathsf{T}_{i+1})$, we get

\begin{align}\footnotesize
\sum_{t=\mathsf{T}_i}^{\mathsf{T}_{i+1}-1} \left( h_{t+1} - h_t \right) 
&\leq \sum_{t=\mathsf{T}_i}^{\mathsf{T}_{i+1}-1} \left( - \frac{\eta_t}{4} \|\hat{\bbg}^{(t)}\|_F^2 + \frac{5}{2} \cdot \frac{\eta_t \sigma}{\sqrt{m}} + \frac{\beta \eta_t^2 n}{2} \right) \nonumber \\
h_{\mathsf{T}_{i+1}} 
&\leq h_{\mathsf{T}_i} - \sum_{t=\mathsf{T}_i}^{\mathsf{T}_{i+1}-1} \left( \frac{\eta_t}{4} \|\hat{\bbg}^{(t)}\|_F^2 \right) + \sum_{t=\mathsf{T}_i}^{\mathsf{T}_{i+1}-1} \left( \frac{5}{2} \cdot \frac{\eta_t \sigma}{\sqrt{m}} + \frac{\beta \eta_t^2 n}{2} \right).
% \label{eqn:sum_descent}
\end{align}

Assume a constant learning rate \(\eta_t = \eta\), and define \(T_i := \mathsf{T}_{i+1} - \mathsf{T}_i\). For each interval \([\mathsf{T}_i, \mathsf{T}_{i+1}-1]\), we have,
\begin{align}
\sum_{t=\mathsf{T}_i}^{\mathsf{T}_{i+1}-1} \frac{\eta}{4} \|\hat{\bbg}^{(t)}\|_F^2 
&\leq h_{\mathsf{T}_i} - h_{\mathsf{T}_{i+1}} 
+ T_i \left( \frac{5}{2} \cdot \frac{\eta \sigma}{\sqrt{m}} + \frac{\beta \eta^2 n}{2} \right). \nonumber
\end{align}

Summing over all \(i = 1, \dots, N\),
\begin{align}
\sum_{i=1}^{N} \sum_{t=\mathsf{T}_i}^{\mathsf{T}_{i+1}-1} \frac{\eta}{4} \|\hat{\bbg}^{(t)}\|_F^2 
&\leq \sum_{i=1}^{N} \left( h_{\mathsf{T}_i} - h_{\mathsf{T}_{i+1}} \right) 
+ \sum_{i=1}^{N} T_i \left( \frac{5}{2} \cdot \frac{\eta \sigma}{\sqrt{m}} + \frac{\beta \eta^2 n}{2} \right) \nonumber \\
&= h_{\mathsf{T}_0} - h_{\mathsf{T}_{N+1}} 
+ T \left( \frac{5}{2} \cdot \frac{\eta \sigma}{\sqrt{m}} + \frac{\beta \eta^2 n}{2} \right), \nonumber
\end{align}
where \(T = \sum_{i=1}^{N} T_i\) is the total number of iterations. Using \(h_{\mathsf{T}_0} - h_{\mathsf{T}_{N+1}} \leq M\), we get,
\begin{align}
\sum_{i=1}^{N} \sum_{t=\mathsf{T}_i}^{\mathsf{T}_{i+1}-1}  \|\hat{\bbg}^{(t)}\|_F^2 
\leq \frac{4M}{\eta} + 2T \left( \frac{5 \sigma}{\sqrt{m}} + \beta \eta n \right). \nonumber
\end{align}

% Therefore, the average squared norm of projected gradients satisfies,
% \begin{align}
% \frac{1}{T} \sum_{t=1}^{T} \|\hat{\bbg}^{(t)}\|_F^2 
% \leq \frac{4M}{\eta T} + 2 \left( \frac{5 \sigma}{\sqrt{m}} + \beta \eta n \right). \nonumber
% \end{align}
Let \(\eta = \sqrt{\frac{2M}{\beta n}} \cdot \frac{1}{\sqrt{T}}\). Then, the average squared norm of the projected gradients satisfies,
\begin{align}
\frac{1}{T} \sum_{t=1}^{T} \|\hat{\bbg}^{(t)}\|_F^2 
&\leq \frac{4M}{\eta T} + 2 \left( \frac{5 \sigma}{\sqrt{m}} + \beta \eta n \right) \nonumber \\
&= \frac{4M}{\left(\sqrt{\frac{2M}{\beta n}} \cdot \frac{1}{\sqrt{T}}\right) T} 
+ 2 \left( \frac{5 \sigma}{\sqrt{m}} + \beta n \cdot \sqrt{\frac{2M}{\beta n}} \cdot \frac{1}{\sqrt{T}} \right) \nonumber \\
&= \frac{4M \sqrt{\beta n}}{\sqrt{2M}} \cdot \frac{1}{T^{3/2}} 
+ 2 \cdot \frac{5 \sigma}{\sqrt{m}} 
+ 2 \beta n \cdot \sqrt{\frac{2M}{\beta n}} \cdot \frac{1}{\sqrt{T}} \nonumber \\
&= \frac{4 \sqrt{2M \beta n}}{T^{3/2}} 
+ \frac{10 \sigma}{\sqrt{m}} 
+ 2 \sqrt{2M \beta n} \cdot \frac{1}{\sqrt{T}}. \nonumber
\end{align}

Thus, the bound becomes,
\begin{align}
\frac{1}{T} \sum_{t=1}^{T} \|\hat{\bbg}^{(t)}\|_F^2 
\leq \frac{4 \sqrt{2M \beta n}}{T^{3/2}} 
+ \frac{10 \sigma}{\sqrt{m}} 
+ \frac{2 \sqrt{2M \beta n}}{\sqrt{T}}. \nonumber
\end{align}
Recall that for $\bbp^{(\mathsf{T}_i)},$ it holds that
\begin{equation}
\|\hat{\bbg}^{(\mathsf{T}_i)} - \bbq^{(\mathsf{T}_i)} \hat{\bbg}^{(\mathsf{T}_i)}\|_F^2 \leq \alpha \|\hat{\bbg}^{(\mathsf{T}_i)}\|_F^2,
\nonumber\end{equation}
for some $\alpha\in(0,1].$
Then,
\begin{align}
\|{\bbq^{(\mathsf{T}_i)}}^\perp \hat{\bbg}^{(\mathsf{T}_i)}\|_F^2 \leq \frac{\alpha}{1 - \alpha} \|\bbq^{(\mathsf{T}_i)} \hat{\bbg}^{(\mathsf{T}_i)}\|_F^2.\nonumber
\end{align}

From Lemma B.3 in \cite{zhao2024galore}, under $\eta \leq \frac{2}{\lambda_{\max}}$, we get,
\begin{align}
\mathbb{E} \|\hat{\bbg}^{(t)}\|_F^2 \leq \mathbb{E} \|\hat{\bbg}^{(\mathsf{T}_i)}\|_F^2, \quad \forall t \in [\mathsf{T}_i, \mathsf{T}_{i+1}).\nonumber
\end{align}

Hence,
\begin{align}
\frac{1}{\mathsf{T}_N} \sum_{i=0}^{N-1} \sum_{t=\mathsf{T}_i}^{\mathsf{T}_{i+1}-1} \|\bbg^{(t)}\|_F^2 
&\leq \frac{1}{\mathsf{T}_N} \sum_{i=0}^{N-1} (\mathsf{T}_{i+1} - \mathsf{T}_i) \|\bbg^{(\mathsf{T}_i)}\|_F^2 \nonumber\\
&\leq \frac{1}{(1-\alpha)\mathsf{T}_N} \sum_{i=0}^{N-1} (\mathsf{T}_{i+1} - \mathsf{T}_i) \|\bbq^{(\mathsf{T}_i)} \hat{\bbg}^{(\mathsf{T}_i)}\|_F^2 \nonumber\\
&\leq \frac{1}{(1-\alpha)\mathsf{T}_N}\left(\frac{4 \sqrt{2M \beta n}}{\mathsf{T}_N^{3/2}} 
+ \frac{10 \sigma}{\sqrt{m}} 
+ \frac{2 \sqrt{2M \beta n}}{\sqrt{\mathsf{T}_N}}\right)\nonumber\\
&\leq\frac{4 \sqrt{2M \beta n}}{(1 - \alpha) \mathsf{T}_N^{5/2}} 
+ \frac{10 \sigma}{(1 - \alpha) \mathsf{T}_N \sqrt{m}} 
+ \frac{2 \sqrt{2M \beta n}}{(1 - \alpha) \mathsf{T}_N^{3/2}}
\end{align}

Accordingly, for any $\varepsilon > 0$, if $\mathsf{T}_N$ satisfies
\[
\frac{4 \sqrt{2M \beta n}}{(1 - \alpha) \mathsf{T}_N^{5/2}} 
+ \frac{10 \sigma}{(1 - \alpha) \mathsf{T}_N \sqrt{m}} 
+ \frac{2 \sqrt{2M \beta n}}{(1 - \alpha) \mathsf{T}_N^{3/2}} \leq \varepsilon,
\]
a sufficient condition for this to hold is,
\[
\mathsf{T}_N \geq \max \left\{
\left( \frac{12 \sqrt{2M \beta n}}{(1 - \alpha) \varepsilon} \right)^{2/5},
\frac{30 \sigma}{(1 - \alpha) \varepsilon \sqrt{m}},
\left( \frac{6 \sqrt{2M \beta n}}{(1 - \alpha) \varepsilon} \right)^{2/3}
\right\},
\]
it follows that
\[
\min_{0 \leq t \leq \mathsf{T}_N} \|\bbg^{(t)}\|_F^2 
\leq \frac{1}{\mathsf{T}_N} \sum_{i=0}^{N-1} \sum_{t=\mathsf{T}_i}^{\mathsf{T}_{i+1}-1} \|\bbg^{(t)}\|_F^2 
\leq \varepsilon.
\]

Thus, there exists an iteration index $t \in [0, \mathsf{T}_N]$ such that
\[
\|\bbg^{(t)}\|_F^2 \leq \varepsilon,
\]
which, by definition, implies that the algorithm reaches an $\varepsilon$-critical point.

This concludes that Algorithm~\ref{alg::SUMO} achieves an $\varepsilon$-critical point.

 \end{proof}

In the following, we provide the auxiliary lemma that is used in the proof of Theorem \ref{thm:Convergence_SUMO}.

\begin{lemma}[Convergence of the Inner Fixed Low-Rank Optimization]\label{lem::ineer_convergance}
Consider the same setting and assumptions as in Theorem~\ref{thm:Convergence_SUMO}. Then, the second time $t = \mathsf{T}_\ell \in \mathbb{N}$ in which Algorithm~\ref{alg::SUMO} enters Block 1 (where it updates the projection matrix) happens for a finite $\ell\in\mathbb{N}$ . 
\end{lemma}

\begin{proof}
By the \(\beta\)-smoothness of \(f\), we have
\[
f(\bbw^{(t+1)}) \leq f(\bbw^{(t)}) + \langle \bbg^{(t)}, \bbw^{(t+1)} - \bbw^{(t)} \rangle + \frac{\beta}{2} \| \bbw^{(t+1)} - \bbw^{(t)} \|_F^2.
\]
Substituting the update rule \( \bbw^{(t+1)} = \bbw^{(t)} - \eta_t \bbo^{(t)} \), we get
\begin{align}
f(\bbw^{(t+1)}) \leq f(\bbw^{(t)}) - \eta_t \langle \bbg^{(t)}, \bbo^{(t)} \rangle + \frac{\beta \eta_t^2}{2} \|\bbo^{(t)}\|_F^2.  \label{eq:smooth}
\end{align}

Since \(\hat{\bbg}^{(t)} = {\mathbf{P}^{(t)}(r)}^\top \bbg^{(t)}\) and \(\bbo^{(t)} \in \text{range}({\mathbf{P}^{(t)}(r)}^\top)\), it holds that
\[
\langle \bbg^{(t)}, \bbo^{(t)} \rangle = \langle \hat{\bbg}^{(t)}, \bbo^{(t)} \rangle.
\]

By equation (2.8) in \citep{li2025note}, we have
\[
-\langle \hat{\bbg}^{(t)}, \bbo^{(t)} \rangle \leq -\frac{1}{4} \|\hat{\bbg}^{(t)}\|_F^2 + \frac{5}{2} \|\hat{\bbg}^{(t)} - \bbm^{(t)}\|_F^2.
\]

Now we bound \(\|\hat{\bbg}^{(t)} - \bbm^{(t)}\|_F\):
\begin{align*}
\|\hat{\bbg}^{(t)} - \bbm^{(t)}\|_F 
&= \|\hat{\bbg}^{(t)} - \nabla f(\bbw^{(t)}) + \nabla f(\bbw^{(t)}) - \bbm^{(t)}\|_F \\
&\leq \|\hat{\bbg}^{(t)} - \nabla f(\bbw^{(t)})\|_F + \|\nabla f(\bbw^{(t)}) - \bbm^{(t)}\|_F \\
&\leq \frac{\sigma}{\sqrt{m}}, \quad \text{(by assumptions \ref{A1}(1)-(2))}.
\end{align*}

Substituting this into the previous inequality gives
\begin{align}\label{upper_bound_OG}
-\langle \hat{\bbg}^{(t)}, \bbo^{(t)} \rangle \leq -\frac{1}{4} \|\hat{\bbg}^{(t)}\|_F^2 + \frac{5}{2} \cdot \frac{\sigma}{\sqrt{m}}.
\end{align}

Substituting into equation \eqref{eq:smooth}, we obtain
\[
f(\bbw^{(t+1)}) \leq f(\bbw^{(t)}) - \frac{\eta_t}{4} \|\hat{\bbg}^{(t)}\|_F^2 + \frac{5}{2} \cdot \frac{\eta_t \sigma}{\sqrt{m}} + \frac{\beta \eta_t^2}{2} \|\bbo^{(t)}\|_F^2.
\]

Since by definition \(\|\bbo^{(t)}\|_F^2 \leq n\), we get
\begin{align}\label{r1}
   f(\bbw^{(t+1)}) \leq f(\bbw^{(t)}) - \frac{\eta_t}{4} \|\hat{\bbg}^{(t)}\|_F^2 + \frac{5}{2} \cdot \frac{\eta_t \sigma}{\sqrt{m}} + \frac{\beta \eta_t^2 n}{2}. 
\end{align}

For constant step size \(\eta_t = \eta\), summing over \(t=1\) to \(T\), we get
\[
f(\bbw^{(T+1)}) \leq f(\bbw^{(1)}) - \frac{\eta}{4} \sum_{t=1}^T \|\hat{\bbg}^{(t)}\|_F^2 + \frac{5}{2} \cdot \frac{\eta \sigma T}{\sqrt{m}} + \frac{\beta \eta^2 n T}{2}.
\]

Rearranging and using \(f(\bbw^{(1)}) - f^* \leq M\), we conclude
\[
\frac{1}{T} \sum_{t=1}^T \|\hat{\bbg}^{(t)}\|_F^2 \leq \frac{4M}{\eta T} + 10 \cdot \frac{\sigma}{\sqrt{m}} + 2 \eta \beta n.
\]
\end{proof}

\section{Additional Information}\label{addditinal_info}
\begin{definition} (Reversibility \citep{tian2021}) \label{Reversibility} A neural network
$\phi$ that maps the input $\boldsymbol{x}$ to output $\boldsymbol{y}=\phi(\boldsymbol{x};\theta)$ is reversible, if there exists $L(\boldsymbol{x} ;\theta)$ so that $\boldsymbol{y}=L(\boldsymbol{x} ;\theta) \boldsymbol{x}$, and the backpropagated gradient $\boldsymbol{g}_{\boldsymbol{x}}$ satisfies $\boldsymbol{g}_{\boldsymbol{x}}=L^{\top}(\boldsymbol{x};\theta) \boldsymbol{g}_{\boldsymbol{y}}$, where $\boldsymbol{g}_{\boldsymbol{y}}$ is the backpropagated gradient at the output $\boldsymbol{y}$. $L(\boldsymbol{x} ;\theta) $ depends on the input $\boldsymbol{x}$ and weight $\theta$ in the network $\phi$.
\end{definition}

Several critical observations regarding Algorithm~\ref{alg::SUMO} warrant attention. Initially, in order to minimize memory consumption, Algorithm \ref{alg::SUMO} implements a per-layer weight update during the process of backpropagation, as advocated by contemporary studies, see, e.g., \cite{lv2024adalomolo}. This approach contrasts with conventional optimizers, which typically update all weights after backpropagation by retaining the complete gradients in memory, a method potentially marked by significant inefficiency. Should there be a desire to generate an adapter (i.e., a parallel low-dimensional LoRA-type model) subsequent to fine-tuning, this can be achieved with efficiency through the following steps. Firstly, the training weights gap $\Delta\triangleq\bbw_\text{Fine-Tuned}-\bbw_\text{Pretrained}$ is computed, where $\bbw_\text{Fine-Tuned}$ denotes the model weight upon process completion, and $\bbw_\text{Pretrained}$ refers to the original model weight. Subsequently, $r_\text{Adaptor} \triangleq \mathsf{rank}(\Delta)$ is determined utilizing a matrix ranking algorithm, followed by the resolution of $\min_{\bba\in\mathbb{R}^{n \times r_\text{Adaptor}}, \bbb\in\mathbb{R}^{r_{\text{Adaptor}} \times m}} \left\|\Delta - \bba \bbb \right\|_F^2$ through any optimization algorithm (e.g., gradient descent). It is noteworthy that any solution to this matrix factorization optimization problem is well-known as a global optimum \citep{NIPS2016_f2fc9902}.

\section{Update Step Rule Formulation} \label{Update_Step}
\begin{definition}\label{def:lor}
[Subspace-Aware Moment-Orthogonalization (SUMO)] SUMO formulates the subsequent gradient update rules. Refer to \begin{footnotesize}
\begin{equation*}
\textbf{SUMO}\left\{
\begin{aligned}
&\hat{\bbg}^{(t)} =  {\bbq^{(t)}}^\top\nabla_\bbw f\left(\bbw_t ;\xi_t\right){\bbr^{(t)}}\quad \\
&\bbm^{(t+1)} = \beta \bbm^{(t)} + (1 - \beta)\hat{\bbg}^{(t)} \\
&\bbo^{(t+1)} = \text{Orthogonalization\_SVD}\left(\bbm^{(t+1)}\right)\\
&\bbw^{(T)}=\bbw^{(0)}+\eta \sum_{t=0}^{T-1}
\Bigl(\bbg^{(t)}-\bbq^{(t)}\bigl(\hat{\bbg}^{(t)}-\bbo^{(t+1)}\bigr){\bbr^{(t)}}^\top\Bigr),
\end{aligned} 
\right.
\end{equation*}
\end{footnotesize} with $\bbq_t \in \mathbb{R}^{m\times r}$ and $\bbr_t \in \mathbb{R}^{r \times n}$ denoting projection matrices, $T\in\mathbb{N}$ representing the subspace update interval, $\eta$ indicating the learning rate, $\xi_t$ constituting a stochastic batch, and $\text{Orthogonalization\_SVD(\bba)}$ as the operator that resolves the following through Singular Value Decomposition (SVD), as described in $$\underset{\bbo}{\operatorname{arg\,min}} \{ \|\bbo - \bba\|_F : \text{either } \bbo^T \bbo = \bbi \text{ or } \bbo\bbo^T = \bbi \}.$$
\end{definition}
\section{Additional Experiments} \label{Additinal_Experiments}

In Table \ref{llama7b}, we evaluated SUMO and state-of-the-art memory-efficient fine-tuning methods on the MAWPS\cite{koncel2015mawps} dataset using the LLaMA2-7B model. We report results across two rank settings (32 and 128), comparing training time, memory usage, and task accuracy. SUMO consistently achieves superior accuracy while maintaining competitive efficiency in both memory and time (comparing to Galore).

\begin{table}[H]
    \centering
    \caption{Fine-tuning LLaMA2-7B on MAWPS\cite{koncel2015mawps}}
    \begin{tabular}{c|c|c|c|c}
    \hline  Methods & Rank & Time(h) $\downarrow$ & Memory (GB) $\downarrow$ & Accuracy (\%) $\uparrow$ \\
    \hline  LoRA & 32 & \textbf{0.40} & 14.36 & 45.80 \\
    \hline  DoRA & 32 & 0.69 & 15.01 & 44.96 \\
    \hline  GaLore & 32 & \fcolorbox{red}{white}{2.59} & 15.15 & 58.40 \\
    \hline SUMO (Newton-Shultz5) & 32 & 1.83 & \textbf{13.86} & \textbf{58.47} \\
    \hline  \textbf{SUMO (SVD)} & 32 & \fcolorbox{green}{white}{1.56} & \textbf{13.86} & \textbf{61.23} \\
     \hline\\
    \hline  LoRA & 128 & \textbf{0.45} & 15.64 & 65.97 \\
    \hline  DoRA & 128 & 0.72 & 16.17 & 66.81 \\
    \hline  GaLore & 128 & \fcolorbox{red}{white}{2.61} & 15.79 & 64.29 \\
    \hline SUMO (Newton-Shultz5) & 128 & 1.78 & \textbf{14.12} & 64.41 \\
    \hline \textbf{SUMO (SVD)} & 128 & \fcolorbox{green}{white}{1.62} & \textbf{14.12} & \textbf{68.03} \\
    \hline
\end{tabular}
\label{llama7b}
\end{table}

% The following table presents complementary comparison to Table~\ref{tab:comparison} in which a full fine-tuning vs. vanilla Muon and our SUMO (mean $\pm$ std when reported).

\paragraph{Comparison with Muon.} Table~\ref{tab:sumo_vs_muon_fullft} below complements Table~\ref{tab:comparison} by comparing full fine-tuning, vanilla Muon, and our SUMO; results are reported as mean ± standard deviation when available.

\begin{table*}[htb]
\centering
\small
\caption{Additional comparison to Table~2: full fine-tuning vs.\ vanilla Muon and our SUMO (mean $\pm$ std when reported).}\label{tab:sumo_vs_muon_fullft}
\resizebox{\linewidth}{!}{
\begin{tabular}{lcccccc}
\toprule
\textbf{Model} & \textbf{Memory} & \textbf{CoLA} & \textbf{STS-B} & \textbf{MRPC} & \textbf{RTE} & \textbf{SST-2} \\
\midrule
Full Fine-Tuning                         & 747M & 62.24             & 90.92             & 91.30             & 79.42             & 94.57 \\
Muon Full Fine Tuning                    & 458M & 61.19             & 90.98             & 92.14             & 80.83             & 94.71 \\
SUMO (Newton-Schulz5, rank=4)            & 197M & 61.81 $\pm$ 0.02  & 90.81 $\pm$ 0.013 & 92.43 $\pm$ 0.034 & 79.33 $\pm$ 0.031 & 94.14 $\pm$ 0.028 \\
SUMO (SVD, rank=4)                       & 197M & 62.32 $\pm$ 0.015 & 91.05 $\pm$ 0.007 & 93.48 $\pm$ 0.022 & 81.08 $\pm$ 0.019 & 94.93 $\pm$ 0.01 \\
SUMO (Newton-Schulz5, rank=8)            & 198M & 61.73 $\pm$ 0.021 & 90.77 $\pm$ 0.032 & 91.93 $\pm$ 0.04  & 79.66 $\pm$ 0.03  & 94.13 $\pm$ 0.025 \\
SUMO (SVD, rank=8)                       & 198M & 61.69 $\pm$ 0.014 & 91.11 $\pm$ 0.02  & 93.72 $\pm$ 0.018 & 81.38 $\pm$ 0.011 & 94.83 $\pm$ 0.01 \\
\bottomrule
\end{tabular}}
\end{table*}

These results show that our SUMO, achieves better performance with a significantly smaller memory footprint compared to Muon full fine-tuning approach.
% \begin{table*}[htb]
% \centering
% \small
% \caption{Zero-shot evaluation of LLaMA-350M models pretrained with sequence length 1024 across reasoning tasks (lower is better for perplexity; higher is better otherwise).}
% \label{tab:zero_shot_reasoning}
% \resizebox{\linewidth}{!}{
% \begin{tabular}{lcccccccccccc}
% \toprule
% \textbf{Method} & \textbf{Memory} & \textbf{Perplexity} & \textbf{BoolQ} & \textbf{RTE} & \textbf{HS} & \textbf{WG} & \textbf{OBQA} & \textbf{ARC-E} & \textbf{ARC-C} & \textbf{PIQA} & \textbf{SciQ} & \textbf{MathQA} & \textbf{Avg.} \\
% \midrule
% AdamW        & 1.37G & 16.30 & 0.4917 & 0.4693 & 0.3688 & 0.5233 & 0.332 & 0.3729 & 0.2449 & 0.6534 & 0.609 & 0.2064 & 0.4272 \\
% APOLLO       & 0.34G & 15.64 & 0.5373 & 0.4698 & 0.3850 & 0.4925 & 0.322 & 0.3788 & 0.2483 & 0.6681 & \textbf{0.624} & 0.2127 & 0.4406 \\
% APOLLO-Mini  & 0.15G & 16.12 & 0.5376 & 0.4562 & 0.3707 & 0.5217 & \textbf{0.324} & 0.3758 & 0.2312 & 0.6638 & 0.619 & 0.2224 & 0.4374 \\
% SUMO         & 0.18G & \textbf{15.49} & \textbf{0.5479} & \textbf{0.4709} & \textbf{0.3937} & \textbf{0.5313} & 0.321 & \textbf{0.3832} & \textbf{0.2496} & \textbf{0.6709} & 0.623 & \textbf{0.2246} & \textbf{0.4416} \\
% \bottomrule
% \end{tabular}}
% \end{table*}

\paragraph{Hyperparameters Grid search.}
To evaluate the impact of Subspace Update Frequency (K) and Ranks (r), we performed a grid search during the pretraining of the LLaMA 130M model on the C4 dataset. This specific setup allows for a direct comparison with the Galore method. Table: Perplexity results from a grid-search of Subspace Update Frequency (K) and Ranks (r) for the LLaMA 130M model pretrained on the C4 dataset. Values are presented as Galore/SUMO.

\begin{table}[htb]
\centering
\small
\caption{Perplexity from a grid-search over Subspace Update Frequency ($K$) and Rank ($r$) for LLaMA-130M on C4. Values are \textbf{Galore/SUMO}}
\label{tab:grid_galore_sumo}
\begin{tabular}{lccc}
\toprule
\textbf{Update Frequency} & \textbf{Rank = 128} & \textbf{Rank = 256} & \textbf{Rank = 512} \\
\midrule
100 & 29.7/\textbf{28.27} & 27.9/\textbf{26.74} & 27.4/\textbf{26.73} \\
250 & 28.1/\textbf{27.86} & 26.5/\textbf{24.87} & 26.2/\textbf{24.82} \\
500 & 27.2/\textbf{25.91} & 25.6/\textbf{24.98} & 25.3/\textbf{24.31} \\
1k  & \textbf{26.8}/25.83  & \textbf{25.1}/25.42  & \textbf{24.8}/24.93  \\
\bottomrule
\end{tabular}
\end{table}

\subsection{Details of benchmarks in Table~\ref{tab:zero_shot_reasoning}}\label{d6}

Specifically, the pretrained models is evaluated on the following tasks:
\begin{itemize}
    \item \textbf{Perplexity:} Measured on the C4 dataset \citep{raffel2020exploring}.
    \item \textbf{Commonsense Reasoning:} BoolQ \citep{clark2019boolq}, RTE \citep{wang2018glue}, HellaSwag (HS) \citep{zellers2019hellaswag}, Winogrande (WG) \citep{sakaguchi2021winogrande}, OpenBookQA (OBQA) \citep{mihaylov2018openbookqa}, ARC-Easy (ARC-E), and ARC-Challenge (ARC-C) \citep{clark2018think}.
    \item \textbf{Physical and Scientific Reasoning:} PIQA \citep{bisk2020piqa}, SciQ \citep{welbl2017crowdsourced}, and MathQA \citep{amini2019mathqa}.
\end{itemize}

\subsection{Details of Fine-Tuning on GLUE}\label{Details_Hyperparameters}
We fine-tune the pre-trained RoBERTa-Base model on the GLUE benchmark using the model provided by the Hugging Face. In Table \ref{tab:my_label_8_} and, we detail the hyper parameters used in fine-tuning. 

\begin{table}[H]
    \centering
\begin{tabular}{c|c|c|c|c|c|c|c|c}
\hline & MNLI & SST-2 & MRPC & CoLA & QNLI & QQP & RTE & STS-B \\
\hline Batch Size & 16 & 16 & 16 & 32 & 16 & 16 & 16 & 16 \\
\hline \# Epochs & 30 & 30 & 30 & 30 & 30 & 30 & 30 & 30 \\
\hline Learning Rate & 1E-05 & 1E-05 & 3E-05 & 3E-05 & 1E-05 & 1E-05 & 1E-05 & 1E-05 \\
\hline Rank Config. & $r=4$& $r=4$& $r=4$& $r=4$ & $r=4$& $r=4$& $r=4$&$r=4$ \\
\hline Projection back scale & 4& 4& 4& 4 & 4& 4& 4&4 \\
\hline Max Seq. Len. & 512& 512& 512& 512 & 512& 512& 512&512 \\
\hline
\end{tabular}
    \caption{Hyperparameters of fine-tuning RoBERTa base for the comparison in Table \ref{tab:comparison} with respect only to rank=4.}
    \label{tab:my_label__}
\end{table}

\begin{table}[H]
    \centering
\begin{tabular}{c|c|c|c|c|c|c|c|c}
\hline & MNLI & SST-2 & MRPC & CoLA & QNLI & QQP & RTE & STS-B \\
\hline Batch Size & 16 & 16 & 16 & 32 & 16 & 16 & 16 & 16 \\
\hline \# Epochs & 30 & 30 & 30 & 30 & 30 & 30 & 30 & 30 \\
\hline Learning Rate & 1E-05 & 2E-05 & 2E-05 & 1E-05 & 1E-05 & 2E-05 & 2E-05 & 3E-05 \\
\hline Rank Config. & $r=8$& $r=8$& $r=8$& $r=8$ & $r=8$& $r=8$& $r=8$&$r=8$ \\
\hline Projection back scale & 2& 2& 2& 2 & 2& 2& 2&2 \\
\hline Max Seq. Len. & 512& 512& 512& 512 & 512& 512& 512& 512\\
\hline
\end{tabular}
    \caption{Hyperparameters of fine-tuning RoBERTa base for the comparison in Table \ref{tab:comparison} with respect only to rank=8.}
    \label{tab:my_label_8_}
\end{table}

\end{document}